\documentclass{article}
\usepackage[T1]{fontenc}
\usepackage{times}
\usepackage{epsfig}
\usepackage{graphicx}
\usepackage{amsmath}
\usepackage{amsthm}
\usepackage{amssymb}
\usepackage{braket}
\usepackage{todonotes}
\usepackage{float}
\usepackage{wrapfig}
\restylefloat{figure}
\usepackage{lipsum}
\usepackage{booktabs}
\usepackage{braket}
\usepackage{xcolor}
\usepackage{geometry}
\usepackage{natbib}

\usepackage{hyperref}
\usepackage{bm}
\usepackage{rotating}
\usepackage{multirow}
\usepackage{subcaption}
\usepackage{booktabs}
\usepackage{adjustbox}
\usepackage{tabularx}
\usepackage{placeins}

\graphicspath{ {./figures/} }
\usepackage{standalone}
\usepackage{pgfplots}

\makeatletter
\let\zz\pgfplots@environment@axis
\def\pgfplots@environment@axis{\endlinechar` \zz}
\makeatother
\theoremstyle{plain}
\newtheorem{theorem}{Theorem}[section]

\theoremstyle{definition}
\newtheorem{definition}{Definition}[section]

\theoremstyle{remark}

\newcommand{\Norm}[1]{\mbox{}\left\|#1\right\|}
\newcommand{\NormS}[1]{\mbox{}\left\|#1\right\|^2}
\newcommand{\setlinespacing}[1]%
           {\setlength{\baselineskip}{#1 \defbaselineskip}}

\newcommand{\BlockDiagk}[1]{\mbox{}\left(%
\begin{array}{cc}
  \Sigma_{k} & \bf{0} \\
  \bf{0} &  \Sigma_{\rho-k}\\
\end{array}\right)}

\newcommand{\BlockDiagkk}[1]{\mbox{}\left(%
\begin{array}{cc}
  \Sigma_{k} & \bf{0} \\
  \bf{0} & \bf{0} \\
\end{array}\right)}

\newcommand{\BlockDiagkrk}[1]{\mbox{}\left(%
\begin{array}{cc}
  \bf{0} & \bf{0} \\
  \bf{0} & \Sigma_{\rho-k} \\
\end{array}\right)}

\newcommand{\BlockDiagkkh}[1]{\mbox{}\left(%
\begin{array}{c}
  \Sigma_{k} \\
  \bf{0} \\
\end{array}\right)}

\newcommand{\BlockDiagkrkh}[1]{\mbox{}\left(%
\begin{array}{c}
  \bf{0} \\
  \Sigma_{\rho-k} \\
\end{array}\right)}

\long\def\killtext#1{}

\RequirePackage{natbib}
\title{Learning Kernels for Structured Prediction using Polynomial Kernel Transformations}
\author{Chetan Tonde, Ahmed Elgammal \\
	Department of Computer Science,
	Rutgers University,\\
	Piscataway, NJ - 08854
	USA. \\
	\texttt{\{cjtonde,elgammal\}@cs.rutgers.edu} \\
}
\begin{document}
\maketitle
\begin{abstract}
Learning the kernel functions used in kernel methods has been a vastly explored area in machine learning. It is now widely accepted that to obtain 'good' performance, learning a kernel function is the key challenge. In this work we focus on learning kernel representations for structured regression. We propose use of polynomials expansion of kernels, referred to as Schoenberg transforms and Gegenbaur transforms, which arise from the seminal result of \citet{Schoenberg:1938fx}. These kernels can be thought of as polynomial combination of input features in a high dimensional reproducing kernel Hilbert space (RKHS).  We learn kernels over input and output for structured data, such that, dependency between kernel features is maximized. We use Hilbert-Schmidt Independence Criterion (HSIC) to measure this. We also give an efficient, matrix decomposition-based algorithm to learn these kernel transformations, and demonstrate state-of-the-art results on several real-world datasets.
\end{abstract}

\section{Introduction}
\label{sec:intro}
The goal in supervised structured prediction is to learn a prediction function $f \colon \mathcal{X} \rightarrow \mathcal{Y}$ from an input domain $\mathcal{X}$, to an output domain $\mathcal{Y}$. As an example, in articulated human pose estimation, input $x \in \mathcal{X}$ would be an image of a human performing an action, and output would be the an interdependent vector of joint positions $(x, y, z)$.  Typically, space of functions $\mathcal{H}$ is fixed (decision trees, neural networks, SVMs) and parametrized. We estimate these parameters from a given set of training examples $S = \{(x_1, y_1), (x_2, x_2), \ldots, (x_m,y_m)\} \subseteq \mathcal{X} \times \mathcal{Y}$, drawn independently identically distributed (i.i.d.) from $P(x,y)$ over $\mathcal{X} \times \mathcal{Y}$. We formulate a meaningful loss function $\mathcal{L} \colon \mathcal{Y} \times \mathcal{Y}$, such as 0-1 loss, or squared loss \citep{Weston:2002ul}, or a structured loss \citep{BenTaskar:2003tt, Tsochantaridis:2004bd,Bo:2010kd,Bratieres:2013vw}.  During prediction for a input $x^* \in \mathcal{X}$, we search for the best possible label $y^*$ so that this loss $\mathcal{L}(f(x^*),y)$ is minimized, given all of training data, and for all possible labels $y \in \mathcal{Y}$, such that,
	\begin{align}
		y^* = f(x^*) = \arg \min_{y \in \mathcal{Y}} \mathcal{L}(f(x^*),y) \nonumber
	\end{align}
	
	In case of kernel methods for structured prediction \citep{Weston:2002ul, BenTaskar:2003tt, Tsochantaridis:2004bd,Bo:2010kd, scholkopf2006predicting, nowozin2011structured}, space of functions $\mathcal{H}$ is specified by positive definite kernel functions, which further are jointly defined on the input and output space as $h((x,y), (x',y'))$. In the most common case, this kernel is factorized over the input and output domain as $k(x,x')$ and $g(y,y')$, with input and output elements $x,x' \in  \mathcal{X} $ and $y,y' \in  \mathcal{Y}$, respectively. These individual kernels map arguments data to reproducing kernel Hilbert space (RKHS) or kernel feature spaces. They are denoted by $\mathcal{K}$ and $\mathcal{G}$, respectively. Now, it is well known that performance of kernel algorithms critically depends on the choice of kernel functions ($k(x,x')$ and $g(y,y')$), and learning them is a challenging problem. In this work, we propose a method to learn kernel feature space representations over input and output data for problem of structured prediction. 
	
	Our contributions are as follows, first, we propose to use of monomial and expansion of dot product and Gegenbaur expansion of radial kernel to learn a polynomial combination kernel features over both input and output. Second, we propose an efficient, matrix-based algorithm to solve for these expansions. And third, we show state-of-the art results on synthetic and real-world datasets data using Twin Gaussian Processes of \cite{Bo:2010kd} as a prototypical kernel method for structured prediction.
	
	\subsection{Related work}
	In the seminal work of \citet{Micchelli:2005vb} showed that we can parametrize a family of kernels $\mathcal{F}$ over a compact set $\Omega$ as
	\[
		\mathcal{F} = 
		\left\lbrace 
				\int_{\Omega} G_{\omega}(x)d\omega\colon p \in \mathcal{P}(\Omega)
		\right\rbrace,
	\]
	where $\mathcal{P}(\Omega)$ is a set of all probability measures on $\Omega$, and $G(\omega)$ is a base kernel parametrized by $\omega$. To illustrate with an example, if we set $\Omega \subseteq \mathbb{R}_+$ and $G_{\omega}(x)$ as multivariate Gaussian kernel over $x$, with variance $\omega$, then $\mathcal{F}$ corresponds to a subset of the class of radial kernels.
	
	Most kernel learning frameworks in past have focussed on learning a single kernel from a family of kernels ($\mathcal{F}$) defined using above equation. All of these frameworks have focussed on  problem of classification or regression. Table \ref{tab:kernel_examples} illustrates previous work on learning kernels for different choices of $\Omega$, and base kernel $G_{\omega}(x,y)$, over data domain $\mathcal{X}$.
	\begin{table}[!htp]
	\centering
		\begin{tabularx}{1\textwidth}{ X | l } 
				\toprule
				Kernel family and base kernel $G_{\omega}(x,y)$ & Related Work \\
				\midrule
				Radial kernels (Iterative), $G_{\omega}(x,y) = e^{-\omega\NormS{x-y}}$  & \citet{Argyriou:2005hb,Argyriou:2006fa}. \\
				\midrule
				Dot product kernels, $G_{\omega}(x,y) = e^{\omega \braket{x,y}}$ & \citet{Argyriou:2005hb, Argyriou:2006fa}. \\
				\midrule
				Finite convex sum of kernels, SimpleMKL & \citet{Rakotomamonjy:2008wf}. \\
				\midrule
				Radial kernels (Semi-infinite Programing, Infinite Kernel Learning) & \citet{gehler2008infinite}\\
				\midrule
				Shift-invariant kernels (Radial  and Anisotropic), $G_{\omega}(x-y) $ & \citet{Shirazi:2010ur} \\
				\bottomrule
		\end{tabularx}
		\caption{Kernel learning frameworks which learn kernels as convex combination of base kernels $G_{\omega}(x,y)$.}
		\label{tab:kernel_examples}	
	\end{table}
	Some of the above approaches work iteratively, \citep{Argyriou:2005hb,Argyriou:2006fa} while others use optimization methods such as, semi-infinite programming \citep{OzogurAkyuz:2010ke}, or QCQP \citep{Shirazi:2010ur}). A review paper by \citet{Gonen:2011ix}  surveys many of these works on various Multiple Kernel Learning algorithms. The relevant works which uses results on polynomial expansion of kernels has been of \citet{Smola:2000tj} which learn dot product kernels using monomial basis $\{1, x, x^2, \ldots\}$, and learning shift-invariant kernels using radial base kernels of \citet{Shirazi:2010ur}, both of these have been proposed for classification and regression.
	
	 We use results on expansion of kernels for learning kernels for problem of structured regression. In propose a framework to learn kernels for structured prediction that use monomial and Gegenbaur expansions to learn a positive combination of base kernels. The Gegenbaur basis has an advantage of being orthonormal and provides a weight parameter $\gamma$ that helps avoid Gibbs phenomenon observed in interpolation \citep{Gottlieb:1992gq}.

	 Outline of this paper is as follows. In section \ref{sec:ker_trans} we describe radial and dot product kernels along and their respective polynomial expansions. In section \ref{sec:hsic} we describe the dependency criteria (HSIC) used in our problem formulation. In section \ref{sec:algo} we describe our proposed algorithm. In section \ref{sec:struct_tgp} we describe Twin Gaussian Processes used for structured prediction. In section \ref{sec:expmts} we present experimental results on several synthetic and real-world datasets. Finally, section \ref{sec:discuss} and section \ref{sec:conc} are discussion and conclusion, respectively.

\section{Kernel Transformations}
\label{sec:ker_trans}
\subsection{Radial and Dot Product Kernels}
A classical result of \citet{Schoenberg:1938fx} states that any continuous, isotropic,  kernel $k(x,x')$, is positive definite if and only if $k(x,x') = \phi(\braket{x,x'})$, and $\phi(t)$ is a real valued function such that
	\[
		\phi(t) = \sum_{i=0}^{\infty} \alpha_k G_k^\lambda (t),\hspace{1cm} t \in [-1,1] 
		\label{eqn:schoenberg}
	\]
where $\sum_{k=0}^{\infty}\alpha_k \geq 0 $, $k \in \mathcal{Z}^+$ and $\alpha_k G_k^\lambda (1) < \infty$. The symbol $G_k^\lambda (t)$ stands for what is known as ultraspherical or Gegenbaur basis polynomials. Examples of such kernels include the Gaussian and Laplacian kernels. 

To understand this better, if we look at a result of \citet{Bochner:uw} which says that, a continuous, shift-invariant kernel $k(\cdot)$ is positive definite, if and only if, it is a inverse Fourier transform of a finite non-negative probability measure $\mu$ on $\mathbb{R}^d$.
	\begin{align*}
        k(x-y) = \phi(z) &= \int_{\mathbb{R}^d} e^{ \sqrt{-1} \braket{z,s} }d\mu(s), & x, y, s \in \mathbb{R}^d \nonumber
	\end{align*}
	This results allows us to represent a shift-invariant kernel uniquely as a spectral distribution in a spectral domain. Now, if we take the Fourier transform of equation \ref{eqn:schoenberg}, we have a unique representation of positive definite kernel $\phi(t)$, and  also a unique representations of positive definite base kernels $G_k^\lambda (t)$, $k=1, 2, \ldots$, in the spectral domain. Hence, the kernel expansion is a nonnegative mixture of base spectral distributions whose kernels are given by kernels $G_k^\lambda (t)$ with nonnegative weights $\alpha_k$'s. 
	
	If we look at monomial basis instead of Gegenbaur basis, a similar interpretation can be given for dot product kernels as mentioned in \citet{Smola:2000tj}. Additionally, we know that the span of monomials form a dense basis in $L_2[-1,1]$, and by Weirerstrass approximation theorem \citep{Szego:1939tn}, any continuous function on a $[-1,1]$ can be approximated uniformly on that interval by polynomials to any degree of accuracy. The Gegenbaur basis is orthonormal and provides additional benefits in interpolation accuracy for functions with sharp changes. Thus avoiding the so called Gibbs phenomenon, \citep{Gottlieb:1992gq}.

In our work, we use above expansions of kernel $k'(\cdot,\cdot)$ on features obtained from an initial kernel $k(\cdot,\cdot)$, defined on $\mathcal{X} \times \mathcal{X}$. We refer to $k(\cdot,\cdot)$ as initial kernel, and estimate $\phi(t)$ which when applied to kernel matrix $[\mathbf{K}]_{i,j}=\braket{x,y}$ gives us a new kernel matrix $[\mathbf{K}]_{i,j}' = \phi([\mathbf{K}]_{i,j}$. Figure~\ref{fig:kerneloneway} illustrates this pictorially.

\begin{figure}
    \centering
    \includegraphics[width = 0.8\textwidth]{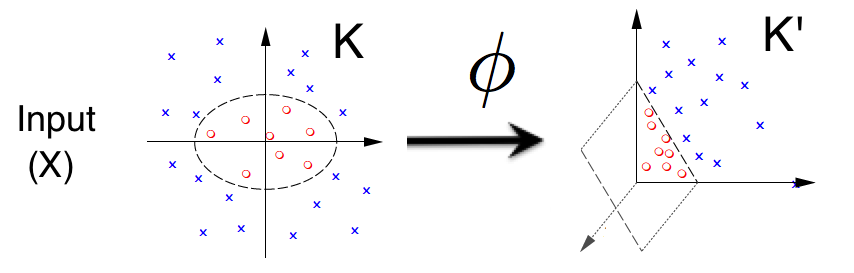}
    \caption{Finding of Kernel Transformations $\phi(\cdot)$}
    \label{fig:kerneloneway}
  \end{figure}

\subsection{Monomial Transformations}
\label{subsec:mono}
In case of the monomial basis functions we have the following expansion.
\begin{theorem}[\citet{Schoenberg:1938fx}]  
	For a continuous function $ \phi \colon [-1,1] \rightarrow \mathbb{R}$ and $k(x,y) = \phi(\braket{x,y})$, the kernel matrix $\mathbf{K}'$ defined as ${\mathbf K'} = \phi([\mathbf{K}]_{i,j})$ is positive definite for any  positive definite matrix $\mathbf{K}$ if and only if $\phi(\cdot)$ is real entire, and of the form below
	\begin{align}
        	k'(\braket{x,y}) = \phi(t) = \sum_{i=0}^{\infty}\alpha_it^i
	        \label{eqn:f_form}
	\end{align}
with $\alpha_i\geq 0$  for all $i \geq 0$.
\end{theorem}
So $\phi(t)$ is an infinitely differentiable and all coefficients $\alpha_j$ are non-negative (eg. $\phi(t)=e^t$).  A benefit of this monomial expansion is that it is easy to compute and can be readily evaluated in parallel.

\subsection{Gegenbaur Transformations}
\label{subsec:gegen}
We obtain Gegenbaur's basis if we insist on having a orthonormal expansion. These expansions are also more general, in the sense that they include a weight function $w(t; \gamma)$, which controls the size of the function space used for representation, controlled by a parameter $\gamma > -1/2$.

To formally state,
\begin{theorem}[\citet{Schoenberg:1938fx}]
	For a real polynomial $\phi \colon  [-1,1] \rightarrow \mathbb{R}$ and for any finite $X = \{x_1, x_2, x_3, \ldots \}$  the matrix $[\phi (\braket{x_i,x_j})]_{i,j}$ is positive definite if and only if $\phi(t)$ is a nonnegative linear combination of Gegenbauer's polynomials $G_i^{\gamma}(t)$, which is,
	\begin{align}
	\phi(t) = \sum_{i=0}^{\infty} \alpha_i G_i^{\gamma}(t)
	\label{eqn:phiG_form}
	\end{align}
\end{theorem}
with $\alpha_i\geq 0$, and $\sum_i a_i G_i^{\gamma}(1)<\infty $.
\begin{definition} Gegenbauer's polynomials are defined as below,
	\begin{align}
	G_0^{\gamma}(t) & =1, 
	G_1^{\gamma}(t) =2\gamma t, 
	\ldots, \\
	G_{i+1}^{\gamma}(t) & =
		\left(
			\frac{2 (\gamma + i)}{i+1} 		\right) t G_{i}^{\gamma}(t)
		- \left( \frac{2\gamma + i-1}{i+1} 	\right) G_{i-1}^{\gamma}(t)
	\end{align}
\end{definition}
As stated earlier $G_k^{\gamma}$ and $G_{l}^{\gamma}$ are orthogonal in $[-1,1]$ and all polynomials are orthogonal with respect to the { weight function } $w(t; \gamma) = (1-t^2)^{(\gamma-1/2)}$ and $\gamma > - 1/2$.
\begin{align} 
\int_{-1}^{1} G_k^{\gamma}(t) G_l^{\gamma}(t) w(t;\gamma)dt = 0 , & \hspace{5mm}  k \neq l.
\end{align}
The weight function above defines a weighted inner product on the space of functions with a norm and a inner product given by
\begin{align}
	\NormS{\phi} = \int_{-1}^{1} w(t; \gamma)\phi(t)\overline{\phi(t)}dt,  \\
	 \braket{\phi, \psi} = \int_{-1}^{1} w(t; \gamma )\phi(t)\overline{\psi(t)}dt
	\label{eqn:gegen_norm}
\end{align}

\begin{wrapfigure}{r}{0.5\textwidth}
	\includegraphics[width=1\linewidth]{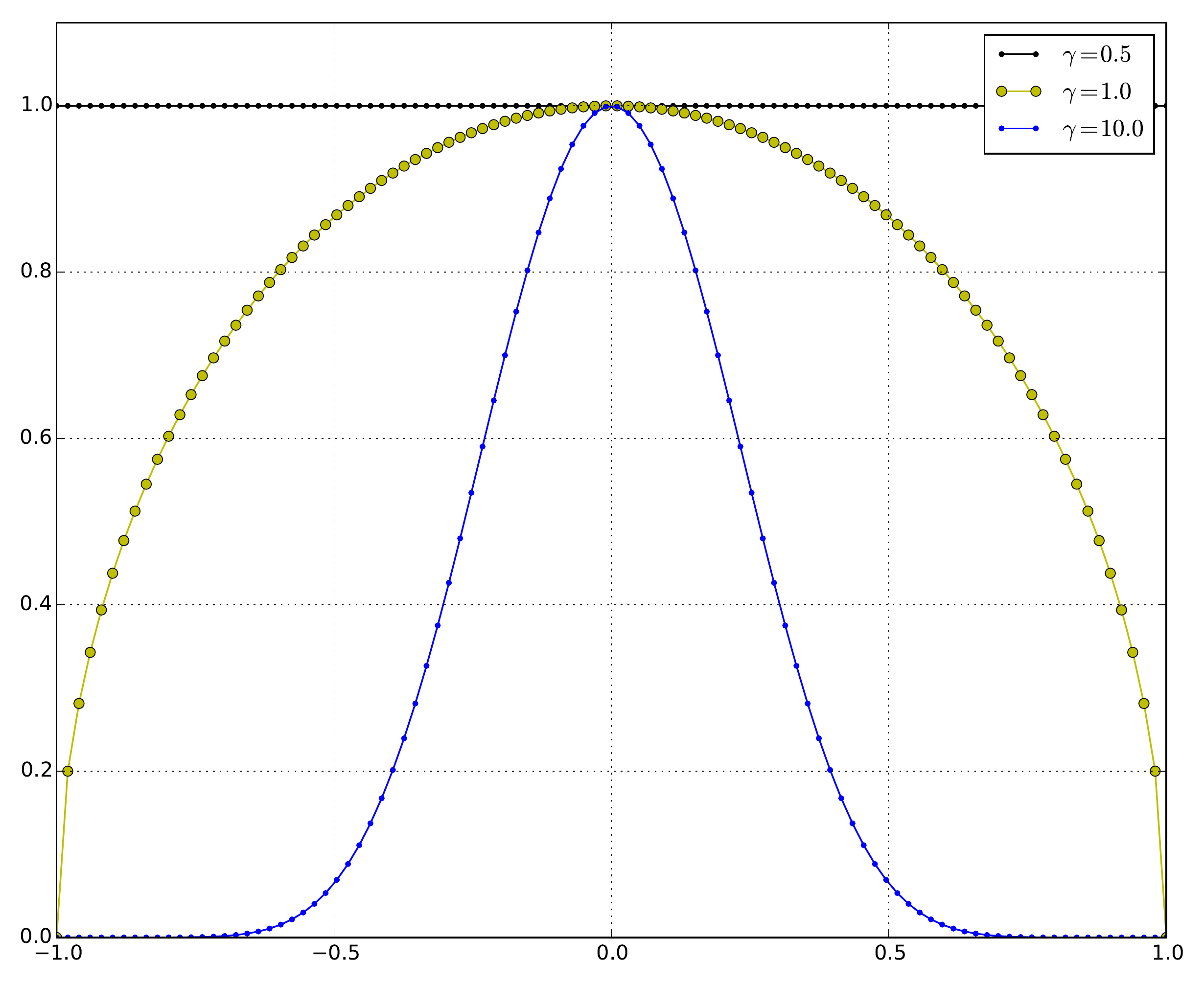}
	\caption{Weight function $w(t;\gamma)$ as  a function of $\gamma$, as we see that the $\gamma$ value controls the behavior of the weight function at the boundaries, and equivalently controlling the size of the function space, which decreases with size.}	
	\label{fig:weight_function}
\end{wrapfigure}

This weight function controls the space of functions over which we estimate these polynomial maps. Figure \ref{fig:weight_function} shows the weight function obtained by different $\gamma$ parameter.  We observe that $\gamma$ value controls the behavior of the function at the boundary points $\Set{1,-1}$. Having this weight function also improves the the quality of interpolation by avoiding the Gibbs phenomenon as further explained in \citet{Gottlieb:1992gq} .

To summarize given a base kernel matrix ${\mathbf K}$ on input data, we expand the target kernel  ${\mathbf K}'$  as a transformation $\phi(t)$ applied to ${\mathbf K}$. This maps the initial RKHS features $k(x,\cdot)$ to new RKHS features $k'(x,\cdot)$. These two forms of the polynomial mapping representations have been known before in the literature but have mostly used for non-structured prediction tasks like regression and classification. In our approach, we use these expansions on both input and output kernels so as to maximize dependence between mapped feature spaces $k'(x,\cdot)$ and $g'(y,\cdot)$, using the Hilbert Schmidt Independence Criterion (HSIC) described below.

\section{Hilbert Schmidt Independence Criterion}
\label{sec:hsic}
To measure cross-correlation or dependence between structured input and output data in kernel feature \citet{Gretton:2005gf} proposed a Hilbert Schmidt Independence criterion (HSIC). Given i.i.d. sampled input-output pair data, $\{(x_1, y_1), (x_2, y_2), \dots, (x_m,y_m)\} \sim P(x,y)$ HSIC measures the dependence between $\mathcal{X}$ and $\mathcal{Y}$.
\begin{definition}
If we have two RKHS's $\mathcal{K}$ and $\mathcal{G}$, then a measure of statistical dependence between $\mathcal{X}$ and $\mathcal{Y}$ is given by the norm of the cross-covariance operator ${\mathbf M}_{xy}: \mathcal{G} \rightarrow \mathcal{K}$, which is defined as,
\begin{align}
        {\mathbf M}_{x y} & := \mathbf{E}_{x,y}[(k(x,\cdot)-\mu_x) \otimes (g(y,\cdot)-\mu_y)] \\
        & =  \mathbf{E}_{x,y}[(k(x,\cdot) \otimes g(y,\cdot))] - \mu_x  \otimes \mu_y
\end{align}
and the measure is given by the Hilbert-Schmidt norm of ${ \mathbf M}_{xy}$ which is,
\begin{align}
        HSIC(p_{xy}, \mathcal{K},\mathcal{G}) := ||{\mathbf M}_{xy}||_{HS}^2 
\end{align}
\end{definition}
So the larger the above norm, the higher the statistical dependence between $\mathcal{X}$ and $\mathcal{Y}$. The advantages of using HSIC for measuring statistical dependence, as stated in  \citet{Gretton:2005gf} are as follows: first, it has good uniform convergence guarantees; second, it has low bias even in high dimensions; and third a number of algorithms can be viewed as maximizing HSIC subject to constraints on the labels/outputs. Empirically, in terms of kernel matrices it is defined as
\begin{definition}
        Let $Z := \{ (x_1, y_1), \ldots (x_m,y_m)\} \subseteq \mathcal{X} \times \mathcal{Y}$ be a series of m independent observations drawn from $p_{xy}$. An unbiased estimator of $HSIC(Z,\mathcal{K},\mathcal{G})$ is given by,
\begin{align}
	HSIC(Z,\mathcal{K},\mathcal{G}) = (m-1)^{-2} trace( {\mathbf K} {\mathbf H}{\mathbf G}{\mathbf H} )
\end{align}
        where ${\mathbf K}, {\mathbf H}, {\mathbf G} \in \mathbb{R}^{m \times m}$, $[{\mathbf K}]_{i,j} := k(x_i,x_j)$, $[{\mathbf G}]_{i,j} := g(y_i,y_j)$ and $[{\mathbf H}]_{i,j} := \delta_{ij} - m^{-1}$
\end{definition}

Now, for well defined {\em normalized } and {\em bounded} kernels, $\mathbf{K}$ and $\mathbf{G}$. We have $HSIC(Z,\mathcal{K},\mathcal{G}) \geq 0$. For the ease of discussion we denote ${HSIC}(Z,\mathcal{K},\mathcal{G})$ by $\overline{HSIC}( {\mathbf K} , {\mathbf G} )$. So we use HSIC to measure dependence between kernels in the target kernel feature space and estimate these transformations (${\alpha_i}'s, \beta_j's$) by maximizing it.

\section{Learning Kernel Transformations}
\label{sec:algo}
The goal in structured prediction is to predict output label $y \in \mathcal{Y}$, given a input example $x \in \mathcal{X}$, our thesis is that if output kernel feature $g'(y,\cdot)$ is more correlated (dependent) with input kernel feature $k'(x,\cdot)$, then we can significantly improve the regression performance.  We propose to use HSIC for this. Also, maximizing HSIC (or equivalently Kernel Target Alignment) has been used as an objective for learning kernels in the past \citep{2012arXiv1203.0550C}. These frameworks maximize the HSIC or alignment between input and outputs, which in our case are structured objects. \citet{shawe2002kernel} in their work provide generalization bounds and which confirm that maximizing alignment (i.e HSIC) does indeed lead to better generalization.

Following on this, if we let $\mathbf{K}' = \phi(\mathbf{K})$, and $\mathbf{G}' = \psi(\mathbf{G})$, where $\mathbf{K}$ and $\mathbf{G}$ are normalized base kernels on input and output data. Using the empirical definition $\overline{HSIC}$ we define the following objective function to optimize as,
\begin{align}
	{\mathbf L} ( {\bm \alpha}*,  {\bm \beta}*) 
	& = \max_{{\bm \alpha},{\bm \beta}} \overline{HSIC}(\phi(\mathbf{K}),\psi(\mathbf{G}))  \\
	\text{subject to\ \ }  & \alpha_i \geq 0 , \beta_j \geq 0 , \forall i,j \geq 0
\label{eqn:main_objective}
\end{align} 
In case of monomial basis as in section \ref{subsec:mono}, we can use equation \ref{eqn:f_form} on kernel matrices $\mathbf{K}$ and $\mathbf{G}$ we get equations for $\phi: {\mathbf K} \rightarrow {\mathbf K'}$ and $\psi: {\mathbf G} \rightarrow {\mathbf G'}$ as follows,
\begin{align}
        \phi(\mathbf{K}) = \sum_{i=0}^{\infty}\alpha_i \mathbf{K}^{(i)}, \alpha_i \geq 0, \ \forall  i \geq 0  \label{eqn:both_maps_i} \\
        \phi(\mathbf{G}) = \sum_{j=0}^{\infty}\beta_j \mathbf{G}^{(j)}, \beta_j\geq 0,\  \forall j \geq 0    \label{eqn:both_maps_o}
\end{align}
where $\mathbf{K}^{(i)}$ is the kernel obtained by applying the $i^{th}$ polynomial basis $t^i$, to the base kernel matrix $\mathbf{K}$ (similarly $G_k^{\lambda}(t)$ for Gegenbaur basis). Figure~\ref{fig:tkl} illustrates this. Assuming our base kernels are both bounded and normalized we also need our new kernels  $\phi({\mathbf K})$ and $\psi({\mathbf G})$ to be at-least bounded. For this purpose, we impose $l_2$-norm regularization constraint on $\alpha_i$'s and $\beta_j$'s, that is $\NormS{{\bm \alpha}}=1$  and $\NormS{{\bm \beta}}=1$. These constraints are similar to the  $l_2$-norm regularization constraint on positive mixture coefficient's in the Multiple Kernel Learning framework of \citet{2012arXiv1203.0550C}. These helps generalization to unseen test data and also avoid arbitrary increase of optimization objective so as to maximize it. 

\begin{figure}
    \centering
    \includegraphics[width = 0.7\textwidth]{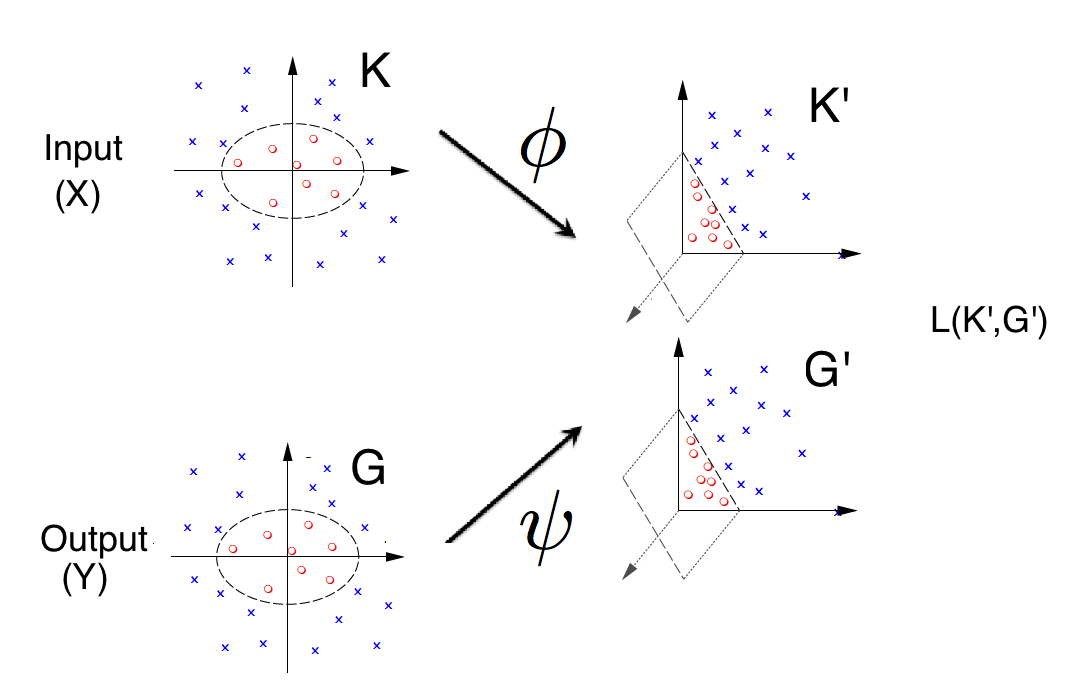}
    \caption{Twin Kernel Learning}
    \label{fig:tkl}
\end{figure}

Substituting the expansions equations \ref{eqn:both_maps_i} and \ref{eqn:both_maps_o} in equation \ref{eqn:main_objective}, and simplifying the main objective \ref{eqn:main_objective} and we get, 
\begin{align}
        \text{maximize } & \sum_{i=0}^{\infty} \sum_{j=0}^{\infty} \alpha_i\beta_j  {\mathbf C}_{i,j}\label{eqn:final} \\
        \text{subject to,  } & \Norm{{\bm \alpha }}_2=1, \Norm{{\bm \beta }}_2=1\nonumber 
\end{align}
where the ${\mathbf C}$-matrix is such that $[\mathbf{C}]_{i,j} = \overline{HSIC}({\mathbf K}^{(i)},{\mathbf G}^{(j)})$. Also, we know that the entries of the ${\mathbf C }$-matrix are non-negative. 
	The $\mathbf{C}$-matrix looks as follows,
			\begin{align}
		        \mathbf{C} & = \begin{array}{c | c c c  c}
			                 	& {K^{(0)}} &  {K^{(1)}} &  {K^{(i)}} & {K^{(d_2)}} \\
			         \hline
			        {G^{(0)}} &  C_{1,1} &  C_{1,2} &  \vdots & C_{d_2,1} \\
			        {G^{(1)}} &  C_{1,1} &  C_{1,2} &  \vdots & C_{d_2,1} \\
		    		{G^{(0)}} & C_{2,1} &  C_{1,2} &  \vdots & C_{2,d_2} \\
		    		{G^{(j)}} & \ldots &  \ldots &  {C_{i,j}} & \ldots \\
		    		{G^{(d_1)}} & C_{d_1,1} &  C_{d_1,2} &  \vdots & C_{d_1,d_2} \\
		        \end{array}
		        \label{fig:Cmatrix}
		\end{align}
		where ${{C}_{i,j}} = \overline{HSIC}({\mathbf K}^{(i)},{\mathbf G}^{(j)})$.

To further explain, every entry $[\mathbf{C}]_{i,j} = \overline{HSIC}({\mathbf K}^{(i)},{\mathbf G}^{(j)})$ represents higher order cross-correlations among the polynomial combination of features of order $i$ and $j$ between input and output, respectively. Hence by appropriately choosing coefficient's $\alpha_i$ and $\beta_j$ we are maximizing these higher order cross-correlations in the target kernel feature space. Here we approximate sum to finite degrees $d_1$ and $d_2$ which leads us to a finite dimensional problem with $deg(\phi) = d_1$ and $deg(\psi)=d_2$. This amounts to using all polynomial combinations of features up to degree $d_1$ for input and $d_2$ for output.  So higher the degree we choose more is the dependence and intuitively better is the prediction. The upper bound on the degree is computationally limited by the added non-linearity in the kernel based prediction algorithm and can lead to overfitting. So we choose these degrees empirically by cross-validation until we saturate the performance of the prediction algorithm.

 We have the following theorem regarding the solution of optimization problem \ref{eqn:final},
\begin{theorem}
	The solution $({\bm \alpha}^*,{\bm \beta}^* )$ to the optimization problem in \ref{eqn:final} is given by the first left and right singular vector of the ${\mathbf C }$-matrix
	\label{thm:soln} 
\end{theorem} 
\begin{proof} Using Perron-Frobenius theorem \citep{Chang:2008tg} for square non-negative matrices $ {\mathbf C}^T{\mathbf C} $ and $ {\mathbf C} {\mathbf C}^T $, we claim that both $ {\mathbf C}^T{\mathbf C} $ and $ {\mathbf C} {\mathbf C}^T $ have Perron vectors ${\bm \alpha}^*$ and ${\bm \beta}^* $, respectively. Both ${\bm \alpha}^*$ and ${\bm \beta}^* $ are the left and right singular vectors of ${\mathbf C}$ and also maximize Eq. \ref{eqn:final}.
\end{proof}
This above theorem gives us our required solution to the problem. Hence, to solve for the unknown's $\alpha_i$'s and $\beta_j$'s we do Singular Value Decomposition (SVD) of the $\mathbf{C}$-matrix and choose ${\bm \alpha} $ and ${\bm \beta} $ to be the first left and right singular vectors Theorem~\ref{thm:soln}. The non-negativity of the ${\bm \alpha}$ and ${\bm \beta}$ vector is guaranteed due to non-negativity of the ${\bm C}$-matrix combined with Perron-Frobenius theoremm \citet{Chang:2008tg}. 

We also observe that $\alpha_0 = \beta_0= 0$., so choosing $d_1 = 1$ corresponds to using the identity mapping $\phi(t) = t$ on the input kernel, which corresponds to using the initial kernel only, or equivalently, no mapping on input kernel. A similarly logic also applies if we set $d_2=1$, then we have $\psi(t) = t$ and no mapping on output base kernel. This is interesting because if we set $d_2=1$ and $d_1$ to be some arbitrary value greater than one, then solution ${\bm \alpha}^*$ is exactly where $\alpha_i^* \propto \overline{HSIC}({\mathbf K}^{(i)},{\mathbf G})$, which same as choosing coefficients based on kernel alignment as in \citet{2012arXiv1203.0550C}.

We also like to point out the similarity of our proposed objective to that of Kernel Canonical Correlation Analysis (KCCA) objective, which also uses HSIC \citep{Chang:2013wq}. In KCCA we find two nonlinear mappings $\phi(\cdot) \in \mathcal{K}$ and $\psi(\cdot) \in \mathcal{G}$ from their  prespecified RKHS's maximizing statistical correlation. In our approach, we also look for analytical kernel transformations $\phi(\cdot)$ and $\psi(\cdot) $ on initial kernel matrices to maximize the same objective.

\section{Twin Gaussian Processes}
\label{sec:struct_tgp}
Twin Gaussian Processes (TGP) of \cite{Bo:2010kd}, are a recent and popular form of structured prediction methods, which model input-output domains using Gaussian processes with covariance functions, represented by $\mathbf{K}$ and $\mathbf{G}$. These covariance matrices encode prior knowledge about the underlying process that is being modeled. In many real world applications data is high dimensional and highly structured, and the choice of kernel functions is not obvious. In our work, we aim to learn kernel covariance matrices simultaneously. We use TGP as an illustrative example to demonstrate the benefits of learning them. Although we note that this framework is not limited only to the use of Twin Gaussian Processes.

In TGP choice of the auxiliary evaluation function is typically some form of information measure, e.g. KL-Divergence or HSIC which are known  to be special cases of Bregman divergences (See \cite{Banerjee:2005vsa}). KL-Divergence is an asymmetric measure of information, while HSIC is symmetric in its arguments. We refer to two versions of Twin Gaussian Processes below corresponding to each of these measures of information. We refer to them as TGP with KL-Divergence or simply TGP, and TGP with HSIC for TGP.

\textbf{TGP with  KL-Divergence}: In this version of TGP, we minimize the KL-divergence between the kernels,  $\bm{K}$ and  $\bm{G}$, given the training data $\mathcal{X} \times \mathcal{Y}$ and test example $x^*$. The prediction function for HOTGP is,
\begin{align}
        y^* = &\arg\min_{y}  D_{KL}(({\bm G}_{Y \cup y}|| {\bm K}_{X \cup x^*})
        \label{eqn1:predkld}
\end{align}

\textbf{TGP with HSIC}: For this version of TGP with HSIC criteria, the prediction function maximizes the HSIC between the kernels $\bm{K}$ and  $\bm{G}$ given the training data $(\mathcal{X} \times \mathcal{Y})$, and test example $x^*$. The prediction function looks as follows, 
\begin{align}
        y^* = &\arg\max_{y}  \overline{HSIC}({\bm G}_{Y \cup y} ,  {\bm K}_{X \cup x^*})
\end{align}

\subsection{Modified Twin Gaussian Processes (TGP)}
We also both of these criteria propose above against degree of mapping $d_1$ and $d_2$. This allows us to show how each information measure is affected as the mapping degrees $d_1$ and $d_2$ are increased. The relationship we observe is straightforward and direct, allowing the choice of $d_1$ and $d_2$ to be made easily. We refer to these new modified TGP's as Higher Order TGP with KL-Divergence (HOTGP)  and Higher Order HSIC (HOHSIC) for TGP using HSIC.

\textbf{Modified TGP with  KL-Divergence}: In this version of TGP, we minimize the KL-divergence between the transformed kernels,  $\phi(\bm{K})$ and  $\psi(\bm{G})$, given the training data$\mathcal{X} \times \mathcal{Y}$, and test example $x^*$. The prediction function for HOTGP is,
\begin{align}
        \mathbf{y}^* = &\arg\min_{y}  D_{KL}((\psi({\bm G}_{Y \cup y})|| \phi({\bm K}_{X \cup x^*}))
\end{align}

\textbf{Modified TGP with HSIC}: For this version of TGP with HSIC criteria, the prediction function maximizes the HSIC between the transformed kernels $\phi(\bm{K})$ and  $\psi(\bm{G})$ given the training data $\mathcal{X} \times \mathcal{Y}$, and test example $x^*$. The prediction function looks as follows, 
\begin{align}
        \mathbf{y}^* = &\arg\max_{y}  \overline{HSIC}((\psi({\bm G}_{Y \cup y}) ,  \phi({\bm K}_{X \cup x^*}))
\end{align}

\section{Experiments}
\label{sec:expmts}
We show empirical results using Twin Gaussain Processes with KL-Divergence and HSIC, and using both monomial and Gegenbaur transformations on synthetic and real-world datasets. To measure improvement in performance over the baseline we look at empirical reduction in error which we call {\em \% Gain } defined as
\[ 
	\%\ Gain = \left (1- \frac{Error_{(mapping)}}{Error_{(no \ mapping)}} \right) \times 100.
\]
In all our experiments, we use Gaussian kernels $k(x_i,x_j) = exp(-\gamma_x||x_i-x_j||^2)$  and $g(y_i,y_j) = exp(-\gamma_y||y_i-y_j||^2)$ as base kernels on input and output, respectively. The bandwidth parameters $\gamma_x$ and $\gamma_y$ were chosen using cross-validation using base kernel on the original dataset. The weight parameters were chosen to be $\lambda_1=0.51$ and $\lambda_2=0.52$ using rough estimates from expressions in \citet{Gottlieb:1992gq} and validated on validation set. For choice of expansion degree we increase degrees $d_1$ and $d_2$ until {\em\% Gain} saturates on the cross-validation. We learn kernel transformations $\phi(\cdot)$  and $\psi(\cdot)$ using above proposed approach. 
\subsection{Datasets}
\subsubsection{Synthetic Data}
\begin{wrapfigure}{r}{0.3\textwidth}
        \includegraphics[width=1\linewidth]{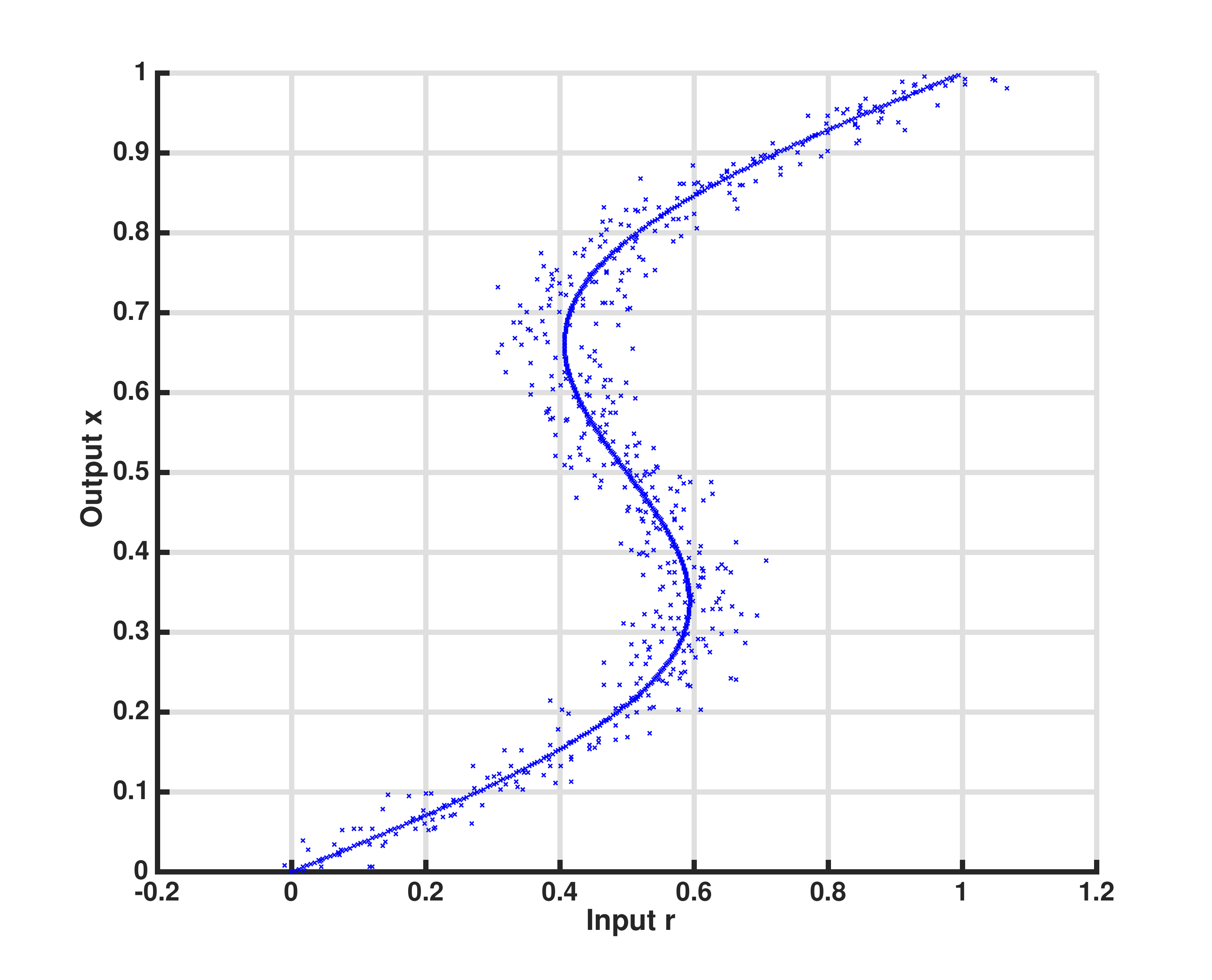}
        \caption{S-Shape  dataset.}
        \label{fig:Sshapedata}
\end{wrapfigure}
\textbf{S Shape Dataset}:  The S-shape synthetic dataset from \citet{Bishop:2002uk} is a simple 1D input/output regression problem. In this dataset, 500 values of inputs ($x$) are sampled uniformly in $(0,1)$, and evaluated for $r = x + 0.3 sin(2\pi x) + \epsilon$ with $\epsilon$ drawn from a zero mean Gaussian noise with standard deviation $\sigma = 0.05$ (Figure~\ref{fig:Sshapedata}). Goal here is to solve the inverse problem which is to predict $x$, given $r$. This dataset is challenging in the sense that it is multivalued (in the middle of the S-shape), discontinuous (at the boundary of uni-valued and multivalued region), and noisy ($\epsilon= \mathcal{N}(0,\sigma)$). The error is metric is used is mean absolute error (MAE).

\textbf{Poser Dataset}: Poser dataset contains synthetic images of human motion capture sequences from Poser 7 \citet{Anonymous:QTlchATe}. These motion sequences includes 8 categories: walk, run, dance, fall, prone, sit, transitions and misc. There are 1927 training examples coming from different sequences of varying lengths and the test set is a continuous sequence of 418 time steps. Input feature vectors are 100d silhouette shape descriptors while output feature vectors are 54d vectors with $x,y$ and $z$ rotation of joint angles. Error metric is mean absolute error (MAE) in mm.
\subsubsection{Real-world data}
\begin{wrapfigure}{r}{0.3\textwidth}
	\centering
	\includegraphics[width=1\linewidth]{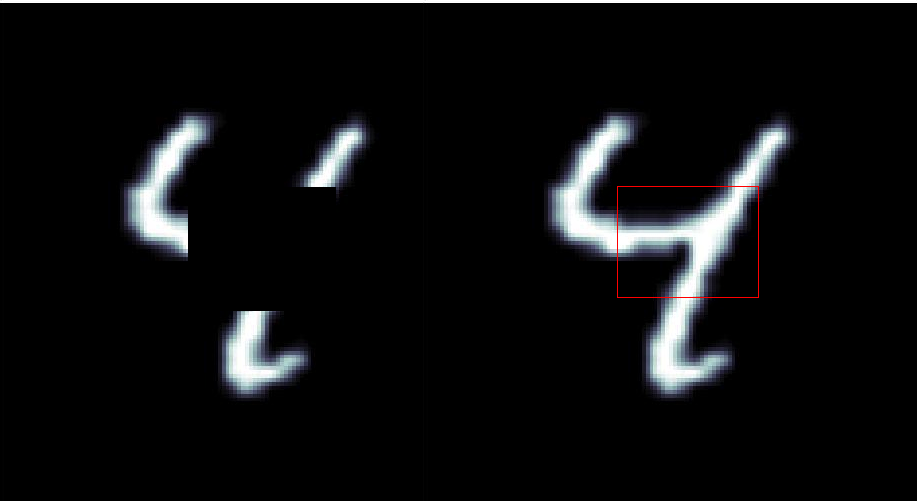}
	\caption{USPS digits.}
	\label{fig:MNIST_pic}
\end{wrapfigure}
\textbf{USPS Handwritten Digits Reconstruction} (Figure~\ref{fig:MNIST_pic}): 
In USPS handwritten digit reconstruction dataset from \citet{Weston:2002ul} our goal is to predict 16 pixel values at center of an image given outer pixels. We use 7425 examples for training (without labels) and 2475 examples (roughly 1/4th for each digit) for testing. Error metric here is reconstruction error measured using mean absolute error (MAE).
\textbf{HumanEva-I Pose Dataset} (Figure~\ref{fig:HumanEvapic}): 
HumanEva-I dataset from \citet{Sigal:2006uw}  is a challenging dataset that contains real motion capture sequences from three different subjects (S1,S2,S3) performing five different actions (Walking, Jogging, Box, Throw/Catch, Gestures).  We train models on all subjects and all actions. We have input images from three different cameras; C1,C2 and C3 and  we use HoG features from \citet{Anonymous:PZi8PH49} on them. The output vectors are 60d with the $x,y,z$ joint positions in mm.  We report results using concatenated features from all three cameras (C1+C2+C2) and also individual features from each individual camera (C1,C2 or C3).

\subsection{Results}
\begin{wrapfigure}{r}{0.3\textwidth}
	\centering
	\includegraphics[width=1\linewidth]{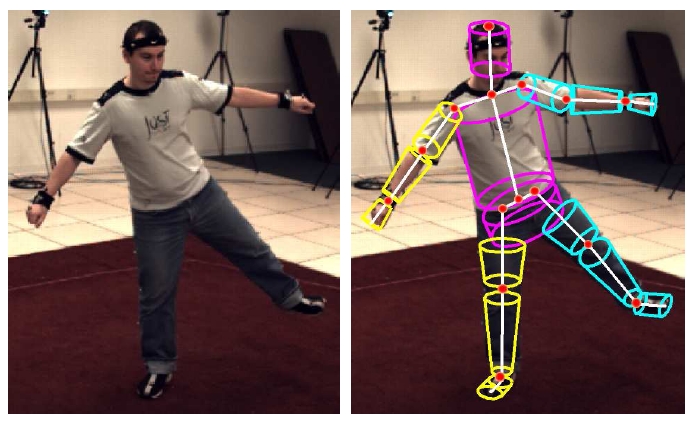}
	\caption{HumanEva-I}
	\label{fig:HumanEvapic}
\end{wrapfigure}
\subsubsection{Synthetic data}
\textbf{S-shape data}: We choose bandwidth parameter to be $\gamma_x=1$ and $\gamma_y=1$ using cross-validation. To illustrate effect of increasing degrees $d_1$ and $d_2$, we run our results on a grid of degrees from the set $\{1,2,3,5,7,11 \}$. In figure~\ref{fig:maps_all} we plot  increase in mapping degree as \% Gain. Figures~\ref{fig:tgp_kl_mono}~and~\ref{fig:tgp_kl_gegen} show results for using KL-Divergence as optimization criteria. Figures~\ref{fig:tgp_hsic_mono}~and~\ref{fig:tgp_hsic_gegen} show results for using HSIC as an optimization criteria. We observe that for both as we increase $d_1$ and $d_1$, {\% Gain} increases i.e. mean absolute error (MAE) reduces. Also for each pair of figures, for each objective, changing from a monomial basis to Gegenbaur basis helps improve {\% Gain} from {\em 31.27\%} to {\em39.49\%} for KL-Div, and  {\em22.31\%} to {\em26.06\%} for HSIC.
\begin{figure}
        \centering
        \begin{subfigure}[b]{0.49\linewidth}
                \includegraphics[width=0.9\linewidth]{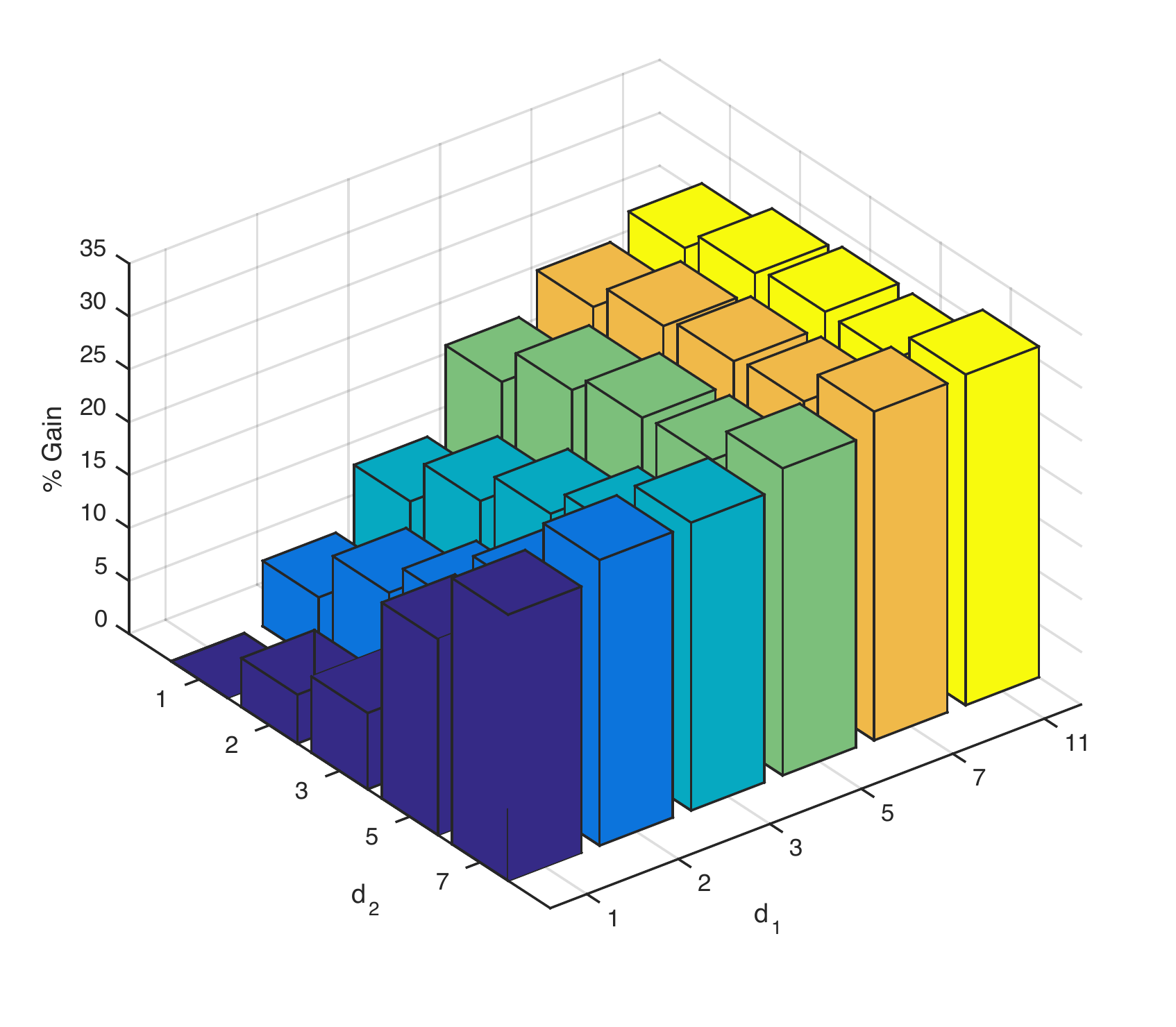}
                \caption{TGP (KL-Div) with Monomial Basis}
                \label{fig:tgp_kl_mono}
        \end{subfigure}%
        \begin{subfigure}[b]{0.49\linewidth}
                \includegraphics[width=0.9\linewidth]{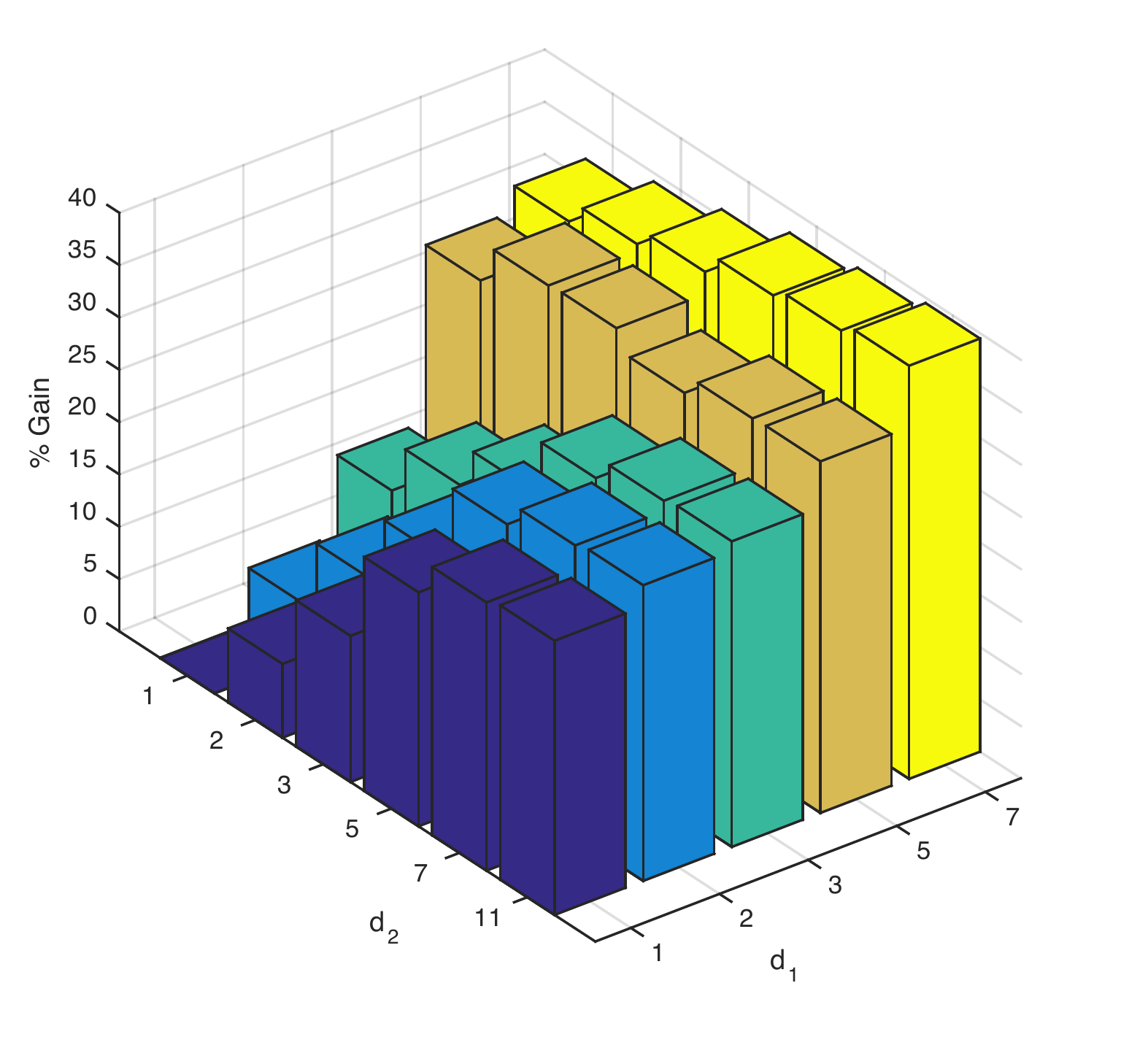}
                \caption{TGP (KL-Div) with Gegenbaur Basis}
                \label{fig:tgp_kl_gegen}
        \end{subfigure} \\
        \begin{subfigure}[b]{0.49\linewidth}
                \includegraphics[width=0.9\linewidth]{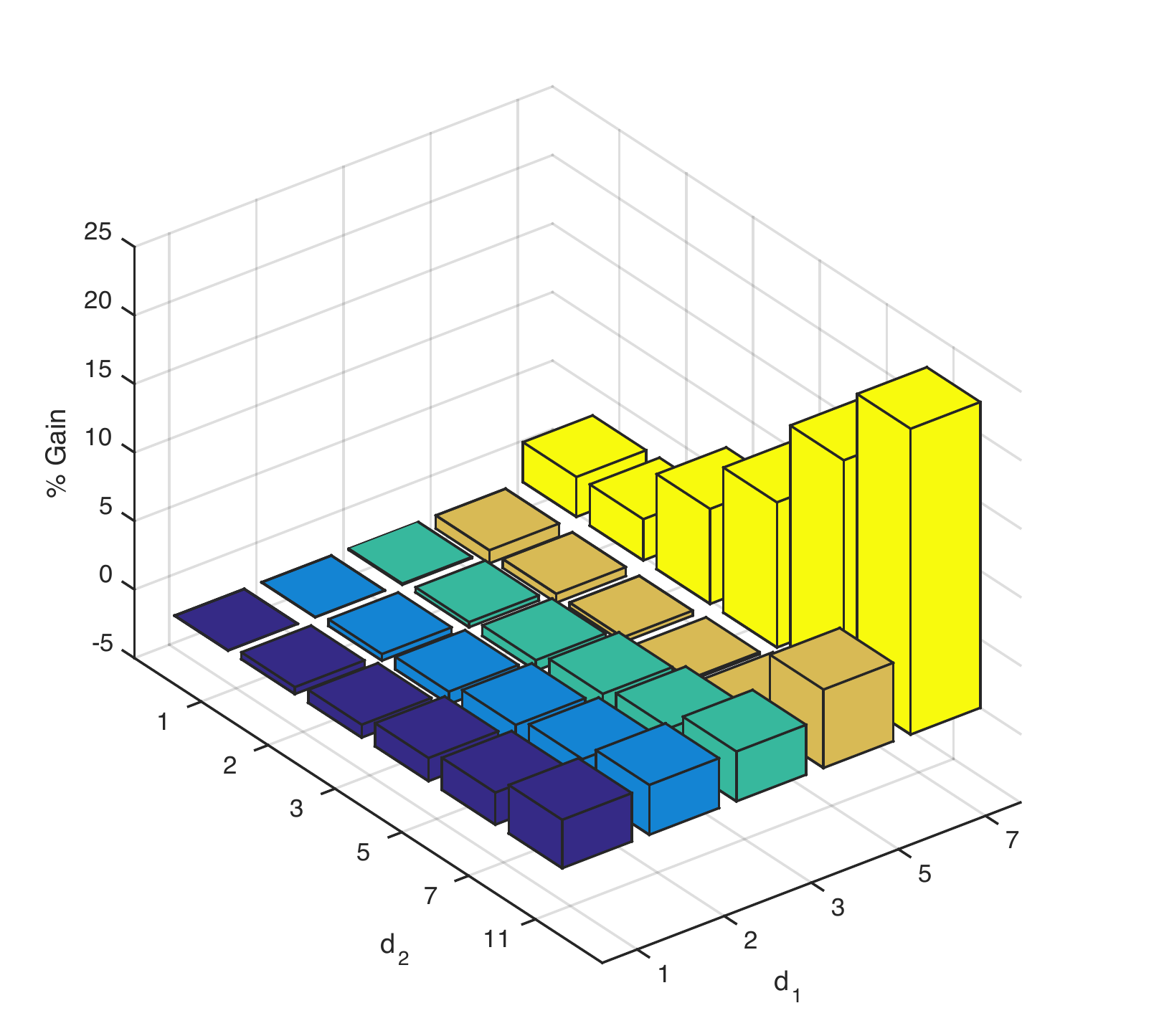}
                \caption{TGP (HSIC) with Monomial Basis}
                \label{fig:tgp_hsic_mono}
        \end{subfigure} 
        \begin{subfigure}[b]{0.49\linewidth}
                \includegraphics[width=0.9\linewidth]{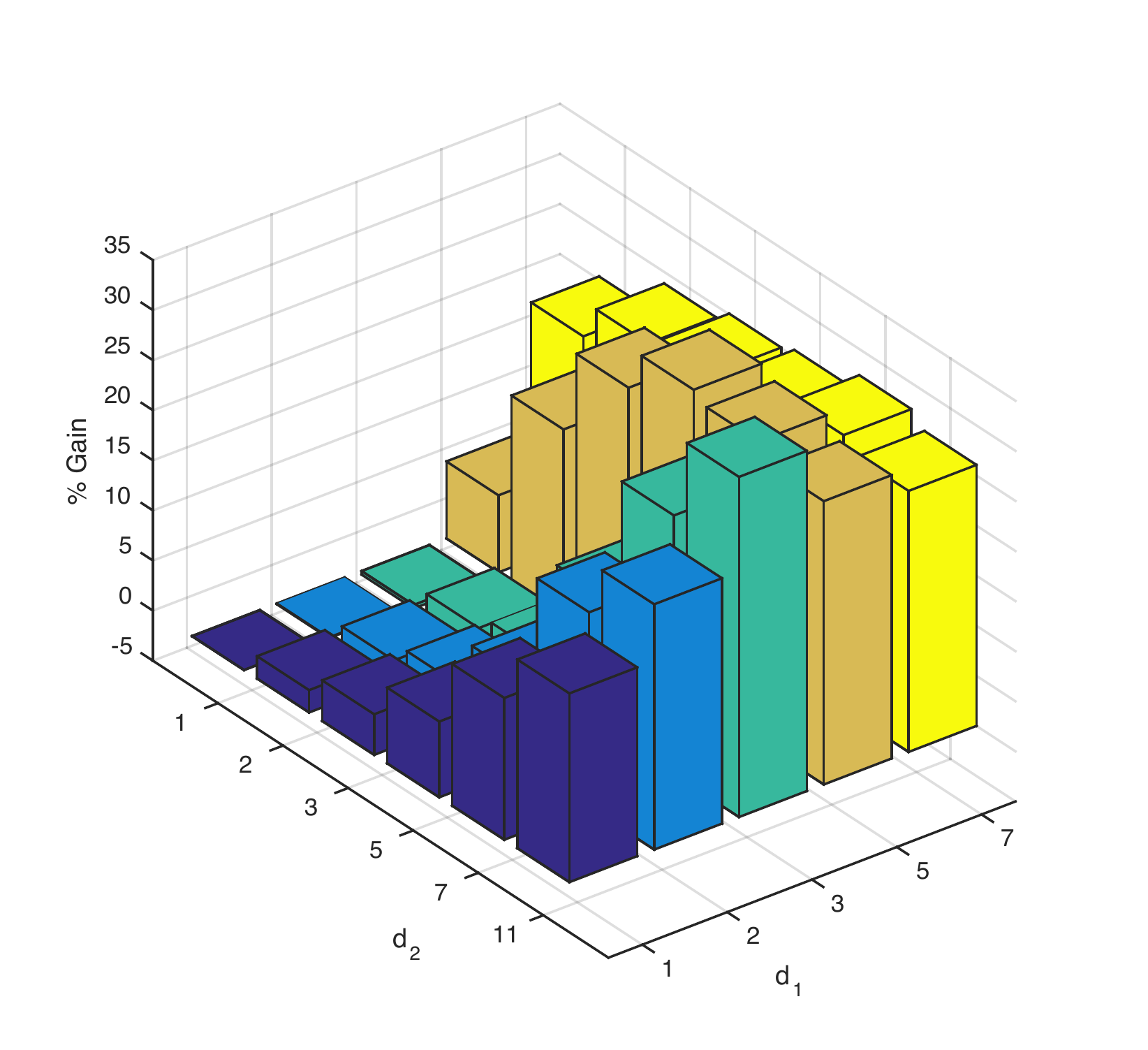}
                \caption{TGP (HSIC) with Gegenbaur Basis}
                \label{fig:tgp_hsic_gegen}
        \end{subfigure}
        \caption{{\em\% Gain} for TGP with KL-Div figures \ref{fig:tgp_kl_mono} and \ref{fig:tgp_kl_gegen} and HSIC (Figures~\ref{fig:tgp_hsic_mono} and \ref{fig:tgp_hsic_gegen}) using monomial and Gegenbaur basis on left and right, respectively.}
        \label{fig:maps_all}
\end{figure}
\begin{table}[!htp]
\centering
\begin{tabular}{ l  c  c c p{1.25cm} }
\toprule
Criterion  & $(d_1,d_2)$ & MAE (w/o map)           & MAE (w/map)     & Gain  \% \\
\midrule
KL-Div & (1,11) & 57.03 & 41.57 & {\bf 27.25}\% \\
\midrule
HSIC & (1,11) 	& 48.08 & 38.71	& {\bf 19.48}\% \\
\midrule
KL-Div (Gegen.) & (1,11) & 57.02 & 44.19 & {\bf  22.50}\%  \\
\midrule
HSIC (Gegen.) & (1,23) & 48.08 & 35.55 & {\bf 36.02}\% \\
\bottomrule
\end{tabular}
\caption{Root Mean Absolute Error for the Poser dataset for the two criteria's of TGP using monomial and Gegenbaur transformation.}
\label{tab:poser}
\end{table}

\textbf{Poser}: For Poser dataset we choose bandwidth to be $\gamma_x=10$ and $\gamma_y=10^{-5}$ using cross-validation. The final results are shown in Table \ref{tab:poser}. In this case we observe that for both basis input degree turns out be $d_1=1$, and using Gegenbaur basis gives us better results of {\em 57.02\%} at $d_2=11$, when compared to monomial basis with  {\em 48.08\%} at $d_2=23$. Here we distinctly observe benefits of using Gegenbaur basis in terms better accuracy and better numerical stability with a lower output degree $d_2$.

\subsubsection{Real-world data}
\textbf{Handwriting Recognition}:
We report our results with bandwidth parameters $\gamma_x=2e10^{-7}$ and $\gamma_y=2e{10}^{-2}$. The mapping degrees were chosen for HOTGP were $(d_1,d_2) = (11,11)$ and for HOHSIC were $(d_1,d_2) = (23,23)$, using cross validation on MAE criteria. Table~\ref{tab:usps} shows summary of results and compares our approach with other kernel-based structured prediction methods. We observe two lowest scores to be from Twin Gaussian process using HSIC with monomial basis, $(d_1,d_2) = (11,11)$, and KL-Divergence with Gegenbaur basis, $(d_1,d_2)=(23,23)$. {\% Gain} shows that using Gegenbaur basis leads to better same results for lower degree over baseline with no-mapping. Best accuracy is obtained for both objective criteria.

\begin{table}
\centering {
\begin{tabular}{ l  c l | l c l  }
\toprule
Method			& $(d_1,d_2)$ &  MAE 	& Method 	& $(d_1,d_2)$ &  MAE\\
\midrule
NN 					& / 			& 0.341 	& KRR 			& / 					& 0.250\\
\hline
SVR 					& / 			& 0.250 	& KDE 			& /					& 0.260\\
\hline
$SOAR_{krr}$ 	& / 			&  0.233 	& $SOAR_{svr}$ & /	& 0.230\\
\hline
HSIC 				& (1,1)		&  0.3399 & HSIC 		& (11,11)			& {\em 0.195} ({\bf 2.50\%}) \\
\hline
KL-Div 				&(1,1)		& 0.2151 & KL-Div &(11,11)		& 0.211 ({2.01})\%\\
\hline
KL-Div Gegen.	&(1,1)		& 0.2101 & KL-Div  Gegen. 		& (11,11)	& {\em0.19 } ({\bf 7.01\%})\\
\hline
HSIC Gegen. 		&(1,1)		&  0.3399 & HSIC Gegen. 			&(23,23)	& {\em0.25442} ({25.14\%})\\
\bottomrule
\end{tabular}}
\caption{ Comparison with others models from \citet{Bo:2009ew} for USPS digits reconstruction dataset. The two lowest errors are {\em emphasized} and their {\em\% Gain} \textbf{bolded}. NN means nearest neighbor regression, KDE means kernel dependency estimation \citep{Weston:2002ul} with 16d latent space obtained by kernel principal component analysis. SOAR means Structured Output Associative regression \citep{Bo:2009ew}. Note: The {\em\% Gain}  show reduction in error compared to no-mapping vs. using mapping for KL-Div and HSIC criteria and with monomial and Gegenbaur basis. }
\label{tab:usps}
\end{table}

\textbf{HumanEva-I Pose Dataset}:
We report results using concatenated features from all three cameras (C1 +C2+C2), and also features from individual camera (C1,C2 and C3). We use Gaussian kernel with $\gamma_x = \gamma_y = 10^{-4}$. For KL-divergence criteria we get $(d_1,d_2)=(1,11)$ for monomial basis, and $(d_1,d_2) = (1,5)$ for Gegenbaur basis. In case of HSIC criteria, we get $(d_1,d_2) = (11,11)$ for both monomial and Gegenbaur basis. Table~\ref{tab:humaevaI}) shows complete set of results. $\%\ Gain$ for each criteria is shown in bold. We observe in both subtables that using concatenated features (C1+C2+C3) gives us better results than using individual camera features. 

In subtable~\ref{tab:humaeva_mono}, we see results with monomials basis for both HSIC KL-Divergence criterion. In general we observe KL-Divergence to be giving better results compared to HSIC. The best results for this case is with features (C1+C2+C3) with KL-Divergence and {\em \% Gain } of $5.080 \%$. In subtable \ref{tab:humaeva_gegen} for the Gegenbaur basis we observe a significant reduction in error when using KL-divergence with  {\em \% Gain} of $99.95\%$. We show consistent improvement in performance with this expansion and much better results using Gegenbaur expansion. We provide detailed results on all subjects and actions in appendix.
\begin{table}[h]
\centering
\begin{subtable}[t]{0.45\textwidth}
	\centering
	\setlength{\tabcolsep}{0.5em}
	\renewcommand{\arraystretch}{1}
	\begin{tabular}{c c c c c }
		\toprule
		Features & Crit.  & w/map               & wo/map       &  \% Gain \\
		\midrule
		\multirow{2}{*}{\begin{tabular}[ c ]{@{}c@{}}HoG\\ (C1C2C3)\end{tabular}} & KL-Div & 45.17 & 42.88 & {\bf 5.08}\%  \\
		\cmidrule{2-5}  & HSIC & 172.66  & 172.59 & {\bf 0.05}\%  \\
		\hline
		\multirow{2}{*}{\begin{tabular}[ c ]{@{}c@{}}HoG\\ (C1)\end{tabular}} & KL-Div & 34.29 & 33.43 & {\bf 2.51}\%  \\
		\cmidrule{2-5}  & HSIC & 172.66 & 173.88 & {\bf 0.08}\%  \\
		\hline
		\multirow{2}{*}{\begin{tabular}[ c ]{@{}c@{}}HoG\\ (C2)\end{tabular}} & KL-Div & 31.99 & 31.58 & {\bf 1.29}\%  \\
		\cmidrule{2-5}  & HSIC & 172.66 & 173.88 & {\bf 0.09}\%  \\
		\hline
		\multirow{2}{*}{\begin{tabular}[ c ]{@{}c@{}}HoG\\ (C3)\end{tabular}} & KL-Div & 30.93 & 30.49 & {\bf 1.41  }\%  \\
		\cmidrule{2-5}  & HSIC &  172.66 & 172.59 & {\bf 0.09}\%  \\
		\bottomrule
	\end{tabular}
	\caption{Monomial transformation:  KL-Div - $(d_1,d_2) = (1,11)$, HSIC-$(d_1,d_2) = (11,11)$}
	\label{tab:humaeva_mono}
 \end{subtable}\hfill
\begin{subtable}[t]{0.45\textwidth}
		\centering
		\setlength{\tabcolsep}{0.5em}
		\renewcommand{\arraystretch}{1}
		\begin{tabular}{c c c c c }
		\toprule
		Features & Crit.  & w/map               & wo/map       &  \% Gain \\
		\midrule
		\multirow{2}{*}{\begin{tabular}[ c ]{@{}c@{}}HoG\\ (C1C2C3)\end{tabular}} & KL-Div 	& 25.40 & 0.011 & {\bf 99.95  }\%  \\
		\cmidrule{2-5}  & HSIC & 172.66 & 	172.29 & {\bf 0.25}\%  \\
		\hline
		\multirow{2}{*}{\begin{tabular}[ c ]{@{}c@{}}HoG\\ (C1)\end{tabular}} & KL-Div & 44.42 & 9.63	&  {\bf 77.46}\%  \\
		\cmidrule{2-5}  & HSIC & 172.66 & 172.28 & {\bf 0.26}\%  \\
		\hline
		\multirow{2}{*}{\begin{tabular}[ c ]{@{}c@{}}HoG\\ (C2)\end{tabular}} & KL-Div & 44.68 & 12.30 & {\bf 71.71}\%  \\
		\cmidrule{2-5}  & HSIC & 172.66  & 172.27 & {\bf 0.26}\%  \\
		\hline
		\multirow{2}{*}{\begin{tabular}[ c ]{@{}c@{}}HoG\\ (C3)\end{tabular}} & KL-Div & 44.35 & 0.01 & {\bf 99.97}\%  \\
		\cmidrule{2-5}  & HSIC & 172.66 & 172.28 & {\bf 0.26}\%  \\
		\bottomrule
	\end{tabular}
	\caption{Gegenbaur transformation: KL-Div-$(d_1,d_2) = (1,5)$, HSIC-$(d_1,d_2) = (11,11)$ }
	\label{tab:humaeva_gegen}	
 \end{subtable}
\caption{ Mean Absolute Error for HumanEva-I dataset for the two criteria KL-Div and HSIC with and without mapping using both monomials transformation and Gegenbaur transformation.}
\label{tab:humaevaI}
\end{table}
 
\section{Discussion} 
\label{sec:discuss}
\begin{table}[!htp]
\centering
\begin{tabular}{ccccp{2cm}l}
\toprule
\setlength{\tabcolsep}{0.2em}
\renewcommand{\arraystretch}{0.5}
{\em \% Gain } Criterion	& S-shape	& Poser	& USPS Digits	&  HumaEva-I (C1+C2+C3) & HumanEva-I (C1,C2,C3)  \\
\toprule
KL-Div.						& 31.27 \%.	& 6.39 \% 		& 1.97 \%		& 5.08\%		&(2.51\%, 1.29\%, 1.40\%)\\
HSIC							& 22.31 \%.	& 1.26\%	& 2.11 \%		& 0.05\%		& (0.08\%, 0.09\%, 0.09\%)\\
\midrule
KL-Div. (Gegen.) 		& 39.49 \%.	& 14.31\%		& 14.31 \%	& 99.95 \% & (77.46\%,71.71\% 99.97\%)\\
HSIC (Gegen.)			& 26.06 \%.	& 7.01\%		& 7.01 \%		& 0.25 \%	& (0.26\%,0.26\%, 0.26\%)\\
\bottomrule
\end{tabular}
\caption{\% Gain for all datasets with both criteria, and using both monomial and Gegenbaur transformation.}
\label{tab:summary}
\end{table}
Table~\ref{tab:summary} provides a complete summary of results for all datasets. It is clear from experimental results that as we increase mapping degrees $d_1$ and $d_2$, we use higher order  combination of polynomial kernel features to maximize dependence between input and output, and this leads to better regression. Reduction in prediction error as indicated by the {\em \% Gain} metric. Choice of degree is done using cross-validation and kernel parameters are selected using the kernel median trick. In S-shape dataset increase in both $d_1$ and $d_2$ helps until the performance saturates, and later falls off due to numerical instability due to the added non-linearity and overfitting. 

For the case of the Poser dataset (Table \ref{tab:poser}) and for HumanEva with KL-Divergence (Table~\ref{tab:humaeva_mono} and \ref{tab:humaeva_gegen}), we see that the best performance is for $d_1$ equal to one. This amounts to choosing choosing the identity mapping/no-mapping on input, $\phi(t) = t$. As we described in section~\ref{sec:algo} this amounts to choosing coefficients proportional to kernel target alignment of \citet{2012arXiv1203.0550C} score between initial input kernel and kernels obtained from basis expansion of initial output kernel $\mathbf{G}$.

The effect of using the two different objective criteria KL-Divergence versus HSIC, we see that in many cases, KL-Divergence does better or as well as HSIC (Table \ref{tab:usps}). In terms of ease of optimization of TGP with HSIC, it turns out to be easier criteria to optimize as there is no explicit training step for it \citep{Bo:2010kd}, and it is relatively easier to compute objective than KL-Divergence.

In terms of choice of basis functions, it is clear that using Gegenbaur basis leads to better numerical stability and often times better results. In some cases like HumaEva with KL-Divergence (Table~\ref{tab:humaeva_gegen}), it does lead to using lower degree values and leads to easier optimization during prediction in TGP.

\section{Conclusions and Future Work}
\label{sec:conc}
We proposed a novel method for learning kernels using polynomial expansions of base kernels. We empirically showed that maximizing dependency between input and output kernel features leads to better performance in structured prediction. We propose an efficient matrix-decomposition based algorithm to learn these kernel transformations. We show state-of-the-art empirical results using Twin Gaussian Processes on several synthetic and real-world datasets. 

For future work, we plan to further investigate; 1) automated learning of kernel parameters $d_1,d_2$ and $\gamma$ by using distributional priors on them and optimizing for data likelihood, 2) extending this framework to multiple kernels for multi-modal and/or multi-task prediction, and 3) joint learning of kernels and prediction.

\bibliographystyle{plainnat}
\bibliography{egbib}
\clearpage
\appendix
\section{Perron-Frobenius}
\begin{theorem}[Perron O., Frobenius G. (1912) ] For $\mathbf{A}_{n \times n} \geq 0$, with spectral radius $r = \rho(\mathbf{A})$, the following statements are true.
	\begin{enumerate}
		\item $r \in \sigma(\mathbf{A})$ and $r > 0$
		\item $r$ is unique and it the spectral radius of $\mathbf{A}$
		\item $\mathbf{A}\mathbf{z}=r\mathbf{z}$ for some $\mathbf{z} \in {\bm \Delta^{n}} = \{\mathbf{x} | \mathbf{x} \geq 0 \text{ with } \mathbf{x \neq 0} \}$
		\item There is unique vector defined by 
		\begin{align}
		        \mathbf{Ap}=r\mathbf{p}, \mathbf{p} >0, \text{ and } ||\mathbf{p}||_1 = 1,
		\end{align}
		is called \textbf{Perron vector} of $\mathbf{A}$ and there are no other nonnegative vectors except for positive multiples of $\mathbf{p}$, regardless of eigenvalue.
	\end{enumerate}
	\label{thm:perron}
\end{theorem}

	\section{Kernel gradients}
	For  data points $X = \{x_1,x_2, \ldots x_m \}$ and test data point $x$.
	\[
		\frac{ \partial \phi( K(x_1, x)) }{  \partial x^{(d)} }
			=  \frac{ \partial \phi (t) }{ \partial t } |_{ t = K(x_1,x) } 	\frac{  \partial K(x_1,x) }{  \partial x^{(d)}  } 
	\]	 
	\[
		\frac{\partial \phi(t) }{ \partial t }
		= \sum_{i=0}^{d_1} \alpha_i H_i^{\gamma} (t)
	\]
	\begin{align}
		H_0^{\gamma} (t) & = 0, H_1^{\gamma} (t)  = 2 \gamma, \\
		H_{i+1}^{\gamma} (t) & = 
			\left( \frac{2 (\gamma + i)}{i+1} \right) \left( tH_{i}^{\gamma} (t) + G_{i}^{\gamma}(t)\right) 
			- \left( \frac{2\gamma + i-1}{i+1} 	\right) H_{i-1}^{\gamma} (t) 
	\end{align}
		\[
		\frac{\partial \phi(K(X,x) ) }{\partial x^{(d)}} = 
		\left[
			\begin{array}{c}
				\frac{\partial \phi (t)}{\partial t} |_{t = K(x_1,x)} \frac{\partial K(x_1,x) }{\partial x^{(d)}}\\
				\frac{\partial \phi (t)}{\partial t} |_{t = K(x_2,x)} \frac{\partial K(x_2,x) }{\partial x^{(d)}}\\
				\vdots \\
				\frac{\partial \phi (t)}{\partial t} |_{t = K(x_m,x)} \frac{\partial K(x_m,x) }{\partial x^{(d)}}\\
			\end{array}
		\right]
	\]
	\subsection{RBF kernel}
	\[
		K(x_i,x_j) = e^{- \gamma \NormS{x_i-x_j}}
	\]
	\[
		\frac{\partial K(X,x) }{\partial x^{(d)}}= 
		\left[ 
			\begin{array}{c}
				- 2\gamma(-x^{(d)}_1+x^{(d)})K(x,x_1) \\
				- 2\gamma(-x^{(d)}_2+x^{(d)})K(x,x_2) \\
				\vdots \\
				- 2\gamma(-x^{(d)}_m+x^{(d)})K(x,x_m)
			\end{array}
		\right]
	\]
	
	\[
		\frac{
				\partial K(x_1, x) 
			}{ 
				\partial x^{(d)} 
			} 
			=  - 2 \gamma  (- x^{(d)}_1 +  x^{(d)} )  K(x,x_1) 
	\]
	\subsection{Linear kernel}
		\[
		K(x_i,x_j) = \gamma \braket{x_i, x_j}
	\]
	\[
		\frac{\partial K(X,x) }{\partial x^{(d)}}= 
		\left[ 
			\begin{array}{c}
				\gamma x^{(d)}_1)\\
				\gamma x^{(d)}_2\\
				\vdots \\
				\gamma x^{(d)}_m
			\end{array}
		\right] 
	\]

\section{Additional Results}
\begin{table}[!htp]
\centering {
\begin{tabular}{ l  c   c  c  c }
\toprule
Crit. / Mean Abs. Er  & (no mapping)           & (mapping)     & Gain  \% \\
\midrule
KL-Div (11,11) & 0.2151 & 0.21078 & {\bf 2.008 }\% \\
\midrule
HSIC (11,11) 	& 0.3399 & 0.19536	& {\bf 2.5007 }\% \\
\midrule
KL-Div (Gegen.) (11,11) & 0.2101 & 0.19536 & {\bf 7.0096 }\%  \\
\midrule
HSIC (Gegen.) (23,23) & 0.3399 & 0.25442 & {\bf  25.1441\% }\% \\
\bottomrule
\end{tabular}}
\label{tab:uspsoursonly}
\caption{Mean Absolute Error for {\em USPS Handwritten digits } dataset for the two criteria, with and without mapping. }
\end{table}

\begin{table}[!htp]
\centering
\begin{sideways}
\setlength{\tabcolsep}{0.25em}
\renewcommand{\arraystretch}{1}
\begin{tabular}{| c | c |c | c | c |c |c |c |c |c |c | }
\toprule
\multirow{2}{*}{Features}  & \multirow{2}{*}{Motions} & \multicolumn{3}{c|}{Subject 1}  & \multicolumn{3}{c|}{Subject 2} & \multicolumn{3}{c|}{Subject 3} \\ 
\cline{3-11}  & & TGP & HOTGP & \% Gain & TGP & HOTGP & \% Gain & TGP & HOTGP & \% Gain \\  \midrule

\multicolumn{1}{|c|}{\multirow{6}{*}{\begin{tabular}[c]{@{}c@{}}HoG\\ (C1C2C3)\end{tabular}}}
& \multicolumn{1}{|c|}{Walking}
&  41.0491   &   38.0943  &  \textbf{ 7.1984}&  27.3651   &   25.1213  &  \textbf{ 8.1998}& 49.5457   &   46.8701  &  \textbf{ 5.4003} \\
& \multicolumn{1}{|c|}{Jog}
&  48.3995   &   47.1128  &  \textbf{ 2.6585}&  37.7046   &   35.4643  &  \textbf{ 5.9417}& 39.8883   &   37.3239  &  \textbf{ 6.4288} \\
& \multicolumn{1}{|c|}{Gestures}
&  16.3853   &   14.5804  &  \textbf{ 11.0155}&  46.6891   &   44.3624  & \textbf{4.9833}& 60.9779   &   59.2251  &  \textbf{ 2.8745} \\
& \multicolumn{1}{|c|}{Box}
&  39.2996   &   37.2232  &  \textbf{ 5.2837}&  48.1626   &   45.4181  & \textbf{5.6983}& 44.3865   &   41.5619  &  \textbf{ 6.3637} \\
& \multicolumn{1}{|c|}{ThrowCatch}
&  84.0748   &   81.9438  &  \textbf{ 2.5346}&  48.492   &   45.9944  & \textbf{5.1504}&  /   &   /  &  / \\
& \multicolumn{1}{|c|}{Average}
 &  45.8417 &  43.7909 &  \textbf{ 5.7381} &  41.6827 &  39.2721 & \textbf{5.9947} & 48.6996 &  46.2453 &  \textbf{ 5.2668}  \\ \midrule
 
\multicolumn{1}{|c|}{\multirow{6}{*}{\begin{tabular}[c]{@{}c@{}}HoG\\ (C1)\end{tabular}}}
& \multicolumn{1}{|c|}{Walking}
&  30.7786   &   29.7702  &  \textbf{3.2765}&  20.2406   &   19.4853  & \textbf{3.7317}& 38.259   &   37.2519  &  \textbf{ 2.6324} \\
& \multicolumn{1}{|c|}{Jog}
&  35.2833   &   34.3592  &  \textbf{2.6191}&  27.5999   &   26.6086  & \textbf{3.5918}& 26.9919   &   26.0356  &  \textbf{ 3.5429} \\
& \multicolumn{1}{|c|}{Gestures}
&  5.3718   &   5.2444  &  \textbf{2.3714}&  39.9867   &   38.8397  & \textbf{2.8685}& 43.4735   &   42.8049  &  \textbf{ 1.5378} \\
& \multicolumn{1}{|c|}{Box}
&  26.9904   &   26.3226  &  \textbf{2.4743}&  40.9108   &   39.764  &  \textbf{2.8032}& 32.2198   &   31.3732  &  \textbf{ 2.6276} \\
& \multicolumn{1}{|c|}{ThrowCatch}
&  70.661   &   69.4619  &  \textbf{1.6971}&  41.2717   &   40.6461  &  \textbf{1.5159}&  /   &   /  &  / \\
& \multicolumn{1}{|c|}{Average}
 &  33.817 &  33.0317 &  \textbf{2.4876} &  34.0019 &  33.0687 &  \textbf{2.9022} & 35.236 &  34.3664 &  \textbf{2.5852}  \\ \midrule
 
\multicolumn{1}{|c|}{\multirow{6}{*}{\begin{tabular}[c]{@{}c@{}}HoG\\ (C2)\end{tabular}}}
& \multicolumn{1}{|c|}{Walking}
&  25.6649   &   24.9178  &  \textbf{2.9112}&  17.4669   &   16.883  &  \textbf{3.3428}& 32.0536   &   31.4839  &  \textbf{1.7774} \\
& \multicolumn{1}{|c|}{Jog}
&  36.1562   &   35.8948  &  \textbf{0.7229}&  25.0174   &   24.4919  &  \textbf{2.1002}& 22.5577   &   22.0753  &  \textbf{2.1384} \\
& \multicolumn{1}{|c|}{Gestures}
&  8.2509   &   7.8924  &  \textbf{4.345}&  41.0384   &   40.7589  &  \textbf{0.6812}& 37.409    &     37.5382   &  \textcolor{red}{-0.3456} \\
& \multicolumn{1}{|c|}{Box}
&  25.952   &   25.4622  &  \textbf{1.8874}&  33.5624   &   33.103  &  \textbf{1.3689}& 35.4483   &   34.9741  &  \textbf{ 1.3379} \\
& \multicolumn{1}{|c|}{ThrowCatch}
&  69.2615   &   68.8563  &  \textbf{0.5849}&  38.06   &   37.7767  &  \textbf{0.7443}&  /   &   /  &  / \\
& \multicolumn{1}{|c|}{Average}
 &  33.0571 &  32.6047 &  \textbf{2.0903} &  31.029 &  30.6027 &  \textbf{1.6475} & 31.8672 &  31.5179 &  \textbf{ 1.227}  \\ \midrule
 
\multicolumn{1}{|c|}{\multirow{6}{*}{\begin{tabular}[c]{@{}c@{}}HoG\\ (C3)\end{tabular}}}
& \multicolumn{1}{|c|}{Walking}
&  26.4685   &   25.6706  &  \textbf{3.0143}&  16.2316   &   15.5688  &  \textbf{4.0835}& 33.3515   &   32.5203  &  \textbf{ 2.4922} \\
& \multicolumn{1}{|c|}{Jog}
&  34.4813   &   34.0262  &  \textbf{1.3199}&  26.0008   &   25.2859  &  \textbf{2.7493}& 22.5078   &   22.0039  &  \textbf{ 2.2388} \\
& \multicolumn{1}{|c|}{Gestures}
&  9.5843   &   9.2134  &  \textbf{ 3.8701}&  36.7243   &   36.3498  &  \textbf{1.0198}& 35.8515    &     37.286   &  \textcolor{red}{-4.0013} \\
& \multicolumn{1}{|c|}{Box}
&  25.9197   &   25.425  &  \textbf{ 1.9087}&  29.6339   &   29.4589  &  \textbf{0.5906}& 30.8706   &   30.269  &  \textbf{ 1.949} \\
& \multicolumn{1}{|c|}{ThrowCatch}
&  67.1347   &   65.7422  &  \textbf{ 2.0743}&  38.2297   &   38.0795  &  \textbf{0.3929}&  /   &   /  &  / \\
& \multicolumn{1}{|c|}{Average}
 &  32.7177 &  32.0155 &  \textbf{ 2.4375} &  29.3641 &  28.9486 &  \textbf{1.7672} & 30.6453 &  30.5198 &  \textbf{ 0.6696}  \\ \midrule
 
\multicolumn{2}{|c|}{Total (\% Gain) }  & \multicolumn{3}{|c}{HoG (C1+C2+C3) :   {\bf5.0796} } & \multicolumn{2}{|c|}{HoG (C1) :  {\bf 2.5147 } }        & \multicolumn{2}{|c|}{HoG (C2) : {\bf 1.2928 } }   & \multicolumn{2}{|c|}{HoG (C3) : {\bf 1.4067}  }  \\ 
\bottomrule

\end{tabular} 
\end{sideways}
 \caption{Evaluation using HoG features on HumanEva-I. Positive {\em \% Gain} for each subject is shown in {\bf bold}, and in \textcolor{red}{red} otherwise. In the table, / shows that the values are not available (no training samples); Average gives the averaged {\em \% Gain} for the different motions of the same subject; C1 means image feature are computed only from the first camera; C1+C2+C3 means image features from three cameras are combined in a single descriptor. Columns TGP and HOTGP indicate the mean absolute error while the {\em \% Gain} column indicates the percentage reduction on error.}
\end{table}

\begin{table}[!htp]
\centering
\begin{sideways}
\setlength{\tabcolsep}{0.25em}
\renewcommand{\arraystretch}{1}
\begin{tabular}{| c | c |c | c | c |c |c |c |c |c |c | }
\toprule
\multirow{2}{*}{Features}  & \multirow{2}{*}{Motions} & \multicolumn{3}{c|}{Subject 1}  & \multicolumn{3}{c|}{Subject 2} & \multicolumn{3}{c|}{Subject 3} \\ 
\cline{3-11}  & & HSIC & HOHSIC & \% Gain & HSIC & HOHSIC & \% Gain & HSIC & HOHSIC & \% Gain \\  \midrule

\multicolumn{1}{|c|}{\multirow{6}{*}{\begin{tabular}[c]{@{}c@{}}HoG\\ (C1C2C3)\end{tabular}}}
& \multicolumn{1}{|c|}{Walking}
 & 159.6824 & 159.406 & 0.1730& 175.6263 & 175.6051 & 0.0120& 200.1092 & 200.1722 & \textcolor{red}{-0.0315} \\
& \multicolumn{1}{|c|}{Jog}
& 182.5472 & 182.6827 & \textcolor{red}{-0.0742} & 196.0942 & 196.2373 & \textcolor{red}{-0.0729} & 204.6334 & 204.7445 & \textcolor{red}{-0.0542} \\
& \multicolumn{1}{|c|}{Gestures}
& 121.9071 & 120.8305 & 0.8831 & 128.1372 & 127.128 & 0.78759 & 161.8518 & 160.9014 & 0.5872 \\
& \multicolumn{1}{|c|}{Box}
& 150.7509 & 150.3133 & 0.29026 & 190.5383 & 189.7172 & 0.4309 & 194.1369 & 193.4358 & 0.3611 \\
& \multicolumn{1}{|c|}{ThrowCatch}
 & 188.2296 & 188.4557 & \textcolor{red}{-0.1201} & 145.4743 & 144.7525 & 0.4961 & / &  / & / \\
& \multicolumn{1}{|c|}{Average}
& 160.6234	& 160.5714	& 0.0839 & 167.1740 & 167.0846 & 0.0808 & 190.1828 & 190.1161 & 0.0888 \\ \midrule
 
\multicolumn{1}{|c|}{\multirow{6}{*}{\begin{tabular}[c]{@{}c@{}}HoG\\ (C1)\end{tabular}}}
& \multicolumn{1}{|c|}{Walking}
 & 159.6824 & 159.4031 & 0.17489 & 175.6263 & 175.6204 & 0.0033 & 200.1092 & 200.1728 & \textcolor{red}{-0.0318} \\
& \multicolumn{1}{|c|}{Jog}
& 182.5472 & 182.7127 & \textcolor{red}{-0.0906} & 196.0942 & 196.2703 & \textcolor{red}{-0.0898} & 204.6334 & 204.7865 & \textcolor{red}{-0.0748} \\
& \multicolumn{1}{|c|}{Gestures}
 & 121.9071 & 120.829 & 0.88431& 128.1372 & 127.0738 & 0.82991 & 161.8518 & 160.9037 & 0.5857 \\
& \multicolumn{1}{|c|}{Box}
& 150.7509 & 150.3093 & 0.29293& 190.5383 & 189.7167 & 0.4312 & 194.1369 & 193.4387 & 0.3596 \\
& \multicolumn{1}{|c|}{ThrowCatch}
& 188.2296 & 188.3374 & \textcolor{red}{-0.0572} & 145.4743 & 144.6998 & 0.53246 & / & / & / \\
& \multicolumn{1}{|c|}{Average}
& 160.6234 & 160.5714	& 0.0839	& 167.1740 & 167.0846 & 0.0808 & 190.1828 & 190.1161 & 0.0888 \\ \midrule
 
\multicolumn{1}{|c|}{\multirow{6}{*}{\begin{tabular}[c]{@{}c@{}}HoG\\ (C2)\end{tabular}}}
& \multicolumn{1}{|c|}{Walking}
 & 159.6824 & 159.4017 & 0.17574 & 175.6263 & 175.6094 & 0.0096 & 200.1092 & 200.1662 & \textcolor{red}{-0.0285} \\
& \multicolumn{1}{|c|}{Jog}
 & 182.5472 & 182.6946 & \textcolor{red}{-0.0807} & 196.0942 & 196.2758 & \textcolor{red}{-0.0926} & 204.6334 & 204.8042 & \textcolor{red}{-0.0834} \\
& \multicolumn{1}{|c|}{Gestures}
& 121.9071 & 120.8318 & 0.88202  & 128.1372 & 127.0976 & 0.81134 & 161.8518 & 160.8771 & 0.6022 \\
& \multicolumn{1}{|c|}{Box}
& 150.7509 & 150.3304 & 0.27895  & 190.5383 & 189.7168 & 0.43115 & 194.1369 & 193.4401 & 0.3589 \\
& \multicolumn{1}{|c|}{ThrowCatch}
& 188.2296 & 188.3753 & \textcolor{red}{-0.0774} & 145.4743 & 144.6909 & 0.53856 & / &  / & /  \\ 
& \multicolumn{1}{|c|}{Average}
& 160.6234	& 160.3267	& 0.2357	& 167.1740 & 166.6781 & 0.3396 & 190.1828 & 189.8219 & 0.2123 \\ \midrule
 
\multicolumn{1}{|c|}{\multirow{6}{*}{\begin{tabular}[c]{@{}c@{}}HoG\\ (C3)\end{tabular}}}
& \multicolumn{1}{|c|}{Walking}
& 159.6824 & 159.4267 & 0.16009 & 175.6263 & 175.6623 & -0.0204 & 200.1092 & 200.2031  & \textcolor{red}{-0.0469} \\
& \multicolumn{1}{|c|}{Jog}
& 182.5472 & 182.744 & \textcolor{red}{-0.1078} & 196.0942 & 196.2974 & -0.10361 & 204.6334 & 204.8073 & \textcolor{red}{-0.0849} \\
& \multicolumn{1}{|c|}{Gestures}
& 121.9071 & 120.8289 & 0.88443 & 128.1372 & 127.0813 & 0.82406 & 161.8518 & 160.9025 & 0.5865 \\
& \multicolumn{1}{|c|}{Box}
& 150.7509 & 150.3118 & 0.29129 & 190.5383 & 189.716 & 0.43159 & 194.1369 & 193.4405 & 0.3587 \\
& \multicolumn{1}{|c|}{ThrowCatch}
& 188.2296 & 188.3351 & \textcolor{red}{-0.0560} & 145.4743 & 144.7051 & 0.52877 & / & / & / \\
& \multicolumn{1}{|c|}{Average}
& 160.6234  & 160.3293 & 0.2343 & 167.1741 & 166.6924  & 0.3321  & 190.1828  & 189.838  & 0.2033 \\ \midrule
 
\multicolumn{2}{|c|}{Total (\% Gain) }  & \multicolumn{3}{|c}{HoG (C1+C2+C3) :   {\bf 0.2566} } & \multicolumn{2}{|c|}{HoG (C1) :  {\bf   0.2589} }        & \multicolumn{2}{|c|}{HoG (C2) : {\bf  0.2589} }   & \multicolumn{2}{|c|}{HoG (C3) : {\bf 0.2625 }  }  \\ 
\bottomrule
\end{tabular} 
\end{sideways}
 \caption{Evaluation using HoG features on HumanEva-I. Positive {\em \% Gain} for each subject is shown in {\bf bold}, and in \textcolor{red}{red} otherwise. In the table, / shows that the values are not available (no training samples); Average gives the averaged {\em \% Gain} for the different motions of the same subject; C1 means image feature are computed only from the first camera; C1+C2+C3 means image features from three cameras are combined in a single descriptor. Columns HSIC and HOHSIC  (Gegen.) indicate the mean absolute error while the {\em \% Gain} column indicates the percentage reduction on error.}
\end{table}


\begin{table}[!htp]
\centering
\begin{sideways}
\setlength{\tabcolsep}{0.25em}
\renewcommand{\arraystretch}{1}
\begin{tabular}{| c | c |c | c | c |c |c |c |c |c |c | }
\toprule
\multirow{2}{*}{Features}  & \multirow{2}{*}{Motions} & \multicolumn{3}{c|}{Subject 1}  & \multicolumn{3}{c|}{Subject 2} & \multicolumn{3}{c|}{Subject 3} \\ 
\cline{3-11}  & & HSIC & HOHSIC & \% Gain & HSIC & HOHSIC & \% Gain & HSIC & HOHSIC & \% Gain \\  \midrule

\multicolumn{1}{|c|}{\multirow{6}{*}{\begin{tabular}[c]{@{}c@{}}HoG\\ (C1C2C3)\end{tabular}}}
& \multicolumn{1}{|c|}{Walking}
& 159.6824 & 159.6333 &  0.03069& 175.6263 & 175.6285 &  -0.0012& 200.1092 & 200.1266 &  \textcolor{red}{-0.0087} \\
& \multicolumn{1}{|c|}{Jog}
& 182.5472 & 182.5801 &  \textcolor{red}{-0.0180}& 196.0942 & 196.1282 &  \textcolor{red}{-0.0174}& 204.6334 & 204.6587 &  \textcolor{red}{-0.0123} \\ 
& \multicolumn{1}{|c|}{Gestures}
& 121.9071 & 121.7042 &  0.1664& 128.1372 & 127.9456 &  0.14957& 161.8518 & 161.6735 &  0.11017 \\
& \multicolumn{1}{|c|}{Box}
& 150.7509 &  150.6684 &  0.0547& 190.5383 & 190.3837 &  0.0811 & 194.1369 & 194.0057 &  0.0676 \\
& \multicolumn{1}{|c|}{ThrowCatch}
& 188.2296 & 188.2713 &  \textcolor{red}{-0.0221}& 145.4743 & 145.3374 &  0.0941&  / &  / &  /  \\
& \multicolumn{1}{|c|}{Average}
& 160.6234	& 160.5714	& 0.0839 & 167.1740 & 167.0846 & 0.0808 & 190.1828 & 190.1161 & 0.0888 \\ \midrule
 
\multicolumn{1}{|c|}{\multirow{6}{*}{\begin{tabular}[c]{@{}c@{}}HoG\\ (C1)\end{tabular}}}
& \multicolumn{1}{|c|}{Walking}
& 159.6824 & 159.6326 &  0.0311& 175.6263 & 175.6299&  \textcolor{red}{-0.0020}& 200.1092 & 200.1262,&  \textcolor{red}{-0.0085} \\
& \multicolumn{1}{|c|}{Jog}
& 121.9071 & 121.704 &  0.1665& 128.1372 & 127.9364 &  0.15667& 161.8518 & 161.6735 &  0.11017 \\
& \multicolumn{1}{|c|}{Gestures}
& 182.5472 & 182.5845 &  \textcolor{red}{-0.0204}& 196.0942 & 196.1347 &  \textcolor{red}{-0.0206}& 204.6334 & 204.6659 &  \textcolor{red}{-0.0158} \\
& \multicolumn{1}{|c|}{Box}
& 150.7509 & 150.6677 &  0.0551& 190.5383 & 190.3838 &  0.08113& 194.1369 & 194.0059 &  0.0674  \\
& \multicolumn{1}{|c|}{ThrowCatch}
& 188.2296 & 188.2513 &  \textcolor{red}{-0.0115}& 145.4743 & 145.3287 &  0.1001&  / &  / &  /  \\
& \multicolumn{1}{|c|}{Average}
& 160.6234 & 160.5680	& 0.0843 & 167.1740 & 167.0827 & 0.1126 & 190.1828 & 194.0059 & 0.0888 \\ \midrule
 
\multicolumn{1}{|c|}{\multirow{6}{*}{\begin{tabular}[c]{@{}c@{}}HoG\\ (C2)\end{tabular}}}
& \multicolumn{1}{|c|}{Walking}
&   159.6824 &  159.6331 &  0.0308&   175.6263 &  175.6278 &  \textcolor{red}{-0.0009}&   200.1092 &  200.1257 &  \textcolor{red}{-0.0082} \\
& \multicolumn{1}{|c|}{Jog}
&   121.9071 &  121.7045 &  0.1661&   128.1372 &  127.9402 &  0.15374&   161.8518 &  161.6687 &  0.11317 \\
& \multicolumn{1}{|c|}{Gestures}
&   182.5472 &  182.581 &  \textcolor{red}{-0.0184}&   196.0942 &  196.1343 &  \textcolor{red}{-0.0204}&   204.6334 &  204.6689 &  \textcolor{red}{-0.0173} \\
& \multicolumn{1}{|c|}{Box}
&   150.7509 &  150.6713 &  0.0527&   190.5383 &  190.3839 &  0.0810&   194.1369 &  194.0059 &  0.0674 \\
& \multicolumn{1}{|c|}{ThrowCatch}
&   188.2296 &  188.2578 &  \textcolor{red}{-0.0149}&   145.4743 &  145.3272 &  0.1011&  / &  / &  / \\
& \multicolumn{1}{|c|}{Average}
 & 160.6234	 & 160.5695	 & 0.0832 & 167.1740  & 167.0826  & 0.1120  & 190.1828  & 190.1173  & 0.0903 \\ \midrule
 
\multicolumn{1}{|c|}{\multirow{6}{*}{\begin{tabular}[c]{@{}c@{}}HoG\\ (C3)\end{tabular}}}
& \multicolumn{1}{|c|}{Walking}
&   159.6824 &  159.6374 &  0.0281&   175.6263 &  175.6382 & \textcolor{red}{-0.0067}&   200.1092 &  200.1319 &  \textcolor{red}{-0.0113} \\
& \multicolumn{1}{|c|}{Jog}
&   121.9071 &  121.704 &  0.1666&   128.1372 &  127.9402 &  0.15374&   161.8518 &  161.6735 &  0.11017 \\
& \multicolumn{1}{|c|}{Gestures}
&   182.5472 &  182.5894 &  \textcolor{red}{-0.0231}&   196.0942 &  196.1389 &  \textcolor{red}{-0.0227}&   204.6334 &  204.6695 &  \textcolor{red}{-0.0176} \\
& \multicolumn{1}{|c|}{Box}
&   150.7509 &  150.6681 &  0.0548&   190.5383 &  190.3837 &  0.08115&   194.1369 &  194.0059 &  0.0674 \\
& \multicolumn{1}{|c|}{ThrowCatch}
&   188.2296 &  188.2506 &  \textcolor{red}{-0.0111}&   145.4743 &  145.3293 &  0.0997&  / &  / &  /  \\
& \multicolumn{1}{|c|}{Average}
& 160.6234	& 160.5699	& 0.0430	& 167.1740 & 167.0860 & 0.0610 & 190.1828 & 190.1202 & 0.0372 \\ \midrule
 
\multicolumn{2}{|c|}{Total (\% Gain) }  & \multicolumn{3}{|c}{HoG (C1+C2+C3) :   {\bf 0.0471} } & \multicolumn{2}{|c|}{HoG (C1) :  {\bf   0.0846} }        & \multicolumn{2}{|c|}{HoG (C2) : {\bf 0.0952 } }   & \multicolumn{2}{|c|}{HoG (C3) : {\bf  0.0952}  }  \\ 
\bottomrule
\end{tabular} 
\end{sideways}
 \caption{Evaluation using HoG features on HumanEva-I. Positive {\em \% Gain} for each subject is shown in {\bf bold}, and in \textcolor{red}{red} otherwise. In the table, / shows that the values are not available (no training samples); Average gives the averaged {\em \% Gain} for the different motions of the same subject; C1 means image feature are computed only from the first camera; C1+C2+C3 means image features from three cameras are combined in a single descriptor. Columns HSIC and HOHSIC indicate the mean absolute error while the {\em \% Gain} column indicates the percentage reduction on error.}
\end{table}

\begin{table}[!htp]
\centering
\begin{sideways}
\setlength{\tabcolsep}{0.25em}
\renewcommand{\arraystretch}{1}
\begin{tabular}{| c | c |c | c | c |c |c |c |c |c |c | }
\toprule
\multirow{2}{*}{Features}  & \multirow{2}{*}{Motions} & \multicolumn{3}{c|}{Subject 1}  & \multicolumn{3}{c|}{Subject 2} & \multicolumn{3}{c|}{Subject 3} \\ 
\cline{3-11}  & & TGP & HOTGP & \% Gain & TGP & HOTGP & \% Gain & TGP & HOTGP & \% Gain \\  \midrule

\multicolumn{1}{|c|}{\multirow{6}{*}{\begin{tabular}[c]{@{}c@{}}HoG\\ (C1C2C3)\end{tabular}}}
& \multicolumn{1}{|c|}{Walking}
& 37.4366 & 9.5482 &  74.4951 & 21.7661 & 6.0762 &  72.084 & 44.4345 & 9.7506 &  78.0562 \\
& \multicolumn{1}{|c|}{Jog}
& 8.6656 & 2.3434 &  72.957  & 60.314 & 11.9035 &  80.2642 & 48.5744 & 9.0458 &  81.3775 \\
& \multicolumn{1}{|c|}{Gestures}
& 49.8854 & 10.5647 &  78.822 & 52.4706 & 10.9267 &  79.1756 & 35.704 & 9.1127 &  74.4771 \\
& \multicolumn{1}{|c|}{Box}
& 38.3083 & 9.5265 &  75.1321 & 52.3877 & 11.4781 &  78.0901 & 48.8305 & 11.0568 &  77.3567 \\ 
& \multicolumn{1}{|c|}{ThrowCatch}
& 63.9985 & 11.78 &  81.5933 & 59.1893 & 11.6716 &  80.2808 & / &  / &  / \\
& \multicolumn{1}{|c|}{Average}
& 39.6588 & 8.7525 &76.5999	 & 49.2255 & 10.4112 & 77.9789 & 44.3858 & 9.7414 & 77.8168 \\ \midrule
 
\multicolumn{1}{|c|}{\multirow{6}{*}{\begin{tabular}[c]{@{}c@{}}HoG\\ (C1)\end{tabular}}}
& \multicolumn{1}{|c|}{Walking}
& 36.9087 & 9.7608 &  73.5541 & 21.6479 & 9.931 &  54.125  & 44.895 & 10.94 &  75.6321 \\
& \multicolumn{1}{|c|}{Jog}
& 9.4207 & 2.5032 &  73.4284 & 60.7753 & 11.8101 &  80.5676  & 51.9379 & 32.982 &  36.4973 \\
& \multicolumn{1}{|c|}{Gestures}
& 49.1918 & 10.6199 &  78.4112  & 52.5459 & 10.9897 &  79.0856  & 35.9421 & 9.2451 &  74.2779 \\
& \multicolumn{1}{|c|}{Box}
& 38.042 & 9.5071 &  75.0089  & 52.8592 & 11.6581 &  77.945  & 48.8239 & 13.3002 &  72.7589 \\
& \multicolumn{1}{|c|}{ThrowCatch}
& 63.0999 & 11.7392 &  81.3959  & 58.7804 & 12.8523 &  78.1351  & / & / &  / \\
& \multicolumn{1}{|c|}{Average}
& 39.3326 & 8.82604	& 76.3597	& 49.3217 & 11.4482 & 73.9716 & 45.3997 & 16.6168 & 64.7915 \\\midrule
 
\multicolumn{1}{|c|}{\multirow{6}{*}{\begin{tabular}[c]{@{}c@{}}HoG\\ (C2)\end{tabular}}}
& \multicolumn{1}{|c|}{Walking}
& 37.8241 & 0.0109 &  99.9711 & 20.1582 & 0.0097 &  99.9518  & 42.5441 & 0.0115 &  99.9729 \\
& \multicolumn{1}{|c|}{Jog}
& 8.8519 & 0.0068&  99.9231  & 59.6685 & 0.0108 &  99.9817  & 51.9699 & 0.0114 &  99.978 \\
& \multicolumn{1}{|c|}{Gestures}
& 48.517 & 0.0118 &  99.9756  & 53.1385 & 0.0100 &  99.9811  & 35.9633 & 0.0112 &  99.9687 \\
& \multicolumn{1}{|c|}{Box}
& 36.7858 & 0.0098 &  99.9733  & 59.3451 & 0.0110 &  99.9813  & 49.0821 & 0.0108 &  99.9779 \\
& \multicolumn{1}{|c|}{ThrowCatch}
& 64.2616 & 0.0141 &  99.9779  & 52.2494 & 0.0097 &  99.9813  & / &  / &  / \\
& \multicolumn{1}{|c|}{Average}
& 39.2480	& 0.0107	& 99.9642	& 48.9119 & 0.0103 & 99.9754 & 44.8898 & 0.0112 & 99.9743 \\ \midrule
 
\multicolumn{1}{|c|}{\multirow{6}{*}{\begin{tabular}[c]{@{}c@{}}HoG\\ (C3)\end{tabular}}}
& \multicolumn{1}{|c|}{Walking}
&  21.4906 &  0.0108 &  99.9497&  12.4445 &  0.0095 &  99.9236 &  25.0639 &  0.0118 &  99.9528 \\
& \multicolumn{1}{|c|}{Jog}
&  4.3108 &  0.0061 &  99.8604&  34.9702 &  0.0103 &  99.9703 &  29.1524 &  0.0115 &  99.9604 \\
& \multicolumn{1}{|c|}{Gestures}
&  28.7849 &  0.0119 &  99.9587&  29.9167 &  0.0097 &  99.9674 &  19.6575 &  0.0112 &  99.9429 \\
& \multicolumn{1}{|c|}{Box}
&  22.2626 &  0.0094 &  99.9578&  30.5953 &  0.0101 &  99.967 &  28.0667 &  0.0099 &  99.9645 \\
& \multicolumn{1}{|c|}{ThrowCatch}
&  35.3846 &  0.0142 &  99.9599&  33.3681 &  0.0103 &  99.969 & / & / & / \\ 
& \multicolumn{1}{|c|}{Average}
& 22.4467	& 0.0104	& 99.9373	& 28.2589	&0.0100	&99.9594	& 25.4851 &0.0111	& 99.9551 \\ \midrule
 
\multicolumn{2}{|c|}{Total (\% Gain) }  & \multicolumn{3}{|c}{HoG (C1+C2+C3) :   {\bf 99.9506} } & \multicolumn{2}{|c|}{HoG (C1) :  {\bf 77.4652} }  & \multicolumn{2}{|c|}{HoG (C2) : {\bf 71.7076} }   & \multicolumn{2}{|c|}{HoG (C3):{\bf  99.9713} }  \\ 
\bottomrule

\end{tabular} 
\end{sideways}
 \caption{Evaluation using HoG features on HumanEva-I. Positive {\em \% Gain} for each subject is shown in {\bf bold}, and in \textcolor{red}{red} otherwise. In the table, / shows that the values are not available (no training samples); Average gives the averaged {\em \% Gain} for the different motions of the same subject; C1 means image feature are computed only from the first camera; C1+C2+C3 means image features from three cameras are combined in a single descriptor. Columns TGP and HOTGP  {\bf (Gegen) (1,5)} indicate the mean absolute error while the {\em \% Gain} column indicates the percentage reduction on error.}
\end{table}

\end{document}